\documentclass{biometrika}

\usepackage{times}
\usepackage{bm}
\usepackage{natbib}

\usepackage{amsmath,mathrsfs,amssymb,graphicx,caption,multicol}

\def\argmax{\mathop{\rm arg\, max}}
\def\argmin{\mathop{\rm arg\, min}}
\def\real{\mathop{{\rm I}\kern-.2em\hbox{\rm R}}\nolimits}
\def\real{{\mathrm R}}
\def\supp{\hbox{\rm supp}}

\def\hhbeta{{\overline \beta}}
\def\hhsigma{{\overline \sigma}}

\def\bw{{w}}\def\bh{{h}}\def\bI{{I}}\def\bu{{u}}\def\bv{{v}}\def\bx{{x}}\def\bX{{X}}\def\by{{y}}
\def\bbeta{{\beta}}\def\hbbeta{{\hbeta}}
\newcommand{\lam}{\lambda}\def\hbeta{\widehat{\beta}}\def\hsigma{\widehat{\sigma}}
\def\sgn{\hbox{\rm sgn}}\def\pa{\partial}\def\scrC{{\mathscr C}}
\def\ep{\varepsilon}\def\eps{\epsilon}\def\Shat{\widehat{S}}
\def\wtil{{\widetilde w}}\def\Ttil{{\widetilde T}}\def\Ctil{{\widetilde C}}
\def\hlam{{\widehat\lam}}

\newcommand{\bel}{\begin{eqnarray}\label}
\newcommand{\eel}{\end{eqnarray}}
\newcommand{\bes}{\begin{eqnarray*}}
\newcommand{\ees}{\end{eqnarray*}}

\begin{document}

\jname{Biometrika}
\jyear{yyyy}
\jvol{0}
\jnum{0}
\accessdate{Advance Access publication on dd mmm yyyy}
\copyrightinfo{\Copyright\ 2008 Biometrika Trust\goodbreak {\em Printed in Great Britain}}

\received{mmm yyyy}
\revised{mmm yyyy}

\markboth{T. Sun \and C.-H. Zhang}{Scaled Sparse Linear Regression}

\title{Scaled Sparse Linear Regression}

\author{Tingni Sun \and Cun-Hui Zhang}
\affil{Department of Statistics and Biostatistics, Hill Center, Busch Campus,
Rutgers University, Piscataway, New Jersey 08854, U.S.A.
\email{tingni@stat.rutgers.edu} \email{czhang@stat.rutgers.edu} }

\maketitle

\begin{abstract}
{Scaled} sparse linear regression jointly estimates the regression coefficients and noise level in a linear model. It chooses an equilibrium with a sparse regression method by iteratively estimating the noise level via the mean residual square and scaling the penalty in proportion to the estimated noise level. The iterative algorithm costs little beyond the computation of a path or grid of the sparse regression estimator for penalty levels above a proper threshold. For the {scaled} lasso, the algorithm is a gradient descent in a convex minimization of a penalized joint loss function for the regression coefficients and noise level. Under mild regularity conditions, we prove that the scaled lasso simultaneously yields an estimator for the noise level and an estimated coefficient vector satisfying certain oracle inequalities for prediction, the estimation of the noise level and the regression coefficients. These inequalities provide sufficient conditions for the consistency and asymptotic normality of the noise level estimator, including certain cases where the number of variables is of greater order than the sample size.
Parallel results are provided for the least squares estimation after model selection by the scaled lasso.
Numerical results demonstrate the superior performance of the proposed methods over an earlier proposal of joint convex minimization.
\end{abstract}


\section{Introduction}
This paper concerns the simultaneous estimation of the regression coefficients and noise level in a high-dimensional linear model. High-dimensional data analysis is a topic of great current interest due to the growth of applications where the number of unknowns far exceeds the number of data points. Among statistical models arising from such applications, linear regression is one of the best understood. Penalization, convex minimization and thresholding methods have been proposed, tested with real and simulated data, and proved to control errors in prediction, estimation and variable selection under various sets of regularity conditions. These methods typically require an appropriate penalty or threshold level. A larger penalty level may lead to a simple model with large bias, while a smaller penalty level may lead to a complex noisy model due to overfitting. Scale-invariance considerations and existing theory suggest that the penalty level should be proportional to the noise level of the regression model. In the absence of knowledge of the latter level, cross-validation is commonly used to determine the former. However, cross-validation is computationally costly and theoretically poorly understood, especially for the purpose of variable selection and the estimation of regression coefficients.
The penalty level selected by cross-validation is called the prediction-oracle in \citet{MeinshausenB06}, who gave an example to show that the prediction-oracle solution does not lead to consistent model selection for the lasso.

Estimation of the noise level in high-dimensional regression is interesting in its own right. Examples include quality control in manufacturing and risk management in finance.

Our study is motivated by \citet{StadlerBG10} and the comments on that paper by \citet{Antoniadis10} and \citet{SunZ10}. \citet{StadlerBG10} proposed to estimate the regression coefficients and noise level by maximizing their joint log-likelihood with an $\ell_1$ penalty on the regression coefficients. Their method has a unique solution due to the joint concavity of the log-likelihood under a certain transformation of the unknown parameters. \citet{SunZ10} proved that this penalized joint maximum likelihood estimator may result in a positive bias for the estimation of the noise level and compared it with two alternatives. The first is a one-step bias correction of the penalized joint maximum likelihood estimator. The second is an iterative algorithm that alternates between estimating the noise level via the mean residual square and scaling the penalty level in a predetermined proportion to the estimated noise level in the lasso or minimax concave penalized selection paths.
In a simulation experiment \citet{SunZ10} demonstrated the superiority of the iterative algorithm, compared with the penalized joint maximum likelihood estimator and its bias correction. However, no theoretical results were given for the iterative algorithm. \citet{Antoniadis10} commented on the same problem from a different perspective by raising the possibility of adding an $\ell_1$ penalty to Huber's concomitant joint loss function. See, for example, section 7.7 of \citet{HuberR09}. Interestingly, the minimizer of this
penalized joint convex loss is identical to the equilibrium of the iterative algorithm for the lasso path. Thus, the convergence of the iterative algorithm is guaranteed by the convexity.

In this paper, we study \citet{SunZ10}'s iterative algorithm for the joint estimation of regression coefficients and the noise level. For the lasso, this is equivalent to jointly minimizing Huber's concomitant loss function with the $\ell_1$ penalty, as \citet{Antoniadis10} pointed out. For simplicity, we call the equilibrium of this algorithm
the scaled version of the penalized regression method, for example the scaled lasso or scaled minimax concave penalized selection, depending on the choice of penalty function. Under mild regularity conditions, we prove
oracle inequalities for prediction and the joint estimation of the noise level and regression coefficients for the {scaled} lasso,
that imply the consistency and asymptotic normality of the scaled lasso estimator for the noise level.
In addition, we prove parallel oracle inequalities for the least squares estimation of the
regression coefficients and noise level after model selection by the scaled lasso.
We report numerical results on the performance of scaled lasso and other scaled penalized methods,
along with that of the corresponding least squares estimator after model selection.
These theoretical and numerical results support the use of the proposed method for high-dimensional regression.

We use the following notation throughout the paper. For a vector $\bv=(v_1,\dots,v_p)$, ${|}\bv{|}_q=(\sum_j|v_j|^q)^{1/q}$ denotes the $\ell_q$ norm with
the usual extensions ${|}\bv{|}_{\infty}=\max_j |v_j|$ and $|\bv|_0=\#\{j:v_j\neq 0\}$.
For design matrices $X$ and subsets $A$ of $\{1,\dots,p\}$,
$x_j$ denotes column vectors of $X$ and $\bX_A$ denotes the matrix composed of columns with indices in $A$.
Moreover, $x_+=\max(x,0)$.

\section{An iterative algorithm}


Suppose we observe a design matrix $\bX=(\bx_1,\ldots,\bx_p)\in\real^{n\times p}$ and a response vector $\by\in\real^n$. For penalty functions $\rho(\cdot)$, consider penalized loss functions of the form
\bel{pen-loss-beta}
L_\lam(\bbeta)
= \frac{{|}\by-\bX\bbeta{|}_2^2}{2n} + \lam^2\sum_{j=1}^p \rho(|\beta_j|/\lam)
\eel
where $\bbeta = (\beta_1,\ldots,\beta_p)'$ is a vector of regression coefficients.
Let the penalty $\rho(t)$ be standardized to ${\dot\rho}(0+)=1$, where
${\dot\rho}(t)=(d/dt)\rho(t)$.
A vector $\hbbeta=(\hbeta_1,\ldots,\hbeta_p)'$
is a critical point of the penalized loss (\ref{pen-loss-beta}) if and only if
\bel{KKT}
\begin{cases}
\bx_j'(\by - \bX\hbbeta)/n = \lam\sgn(\hbeta_j){\dot\rho}(|\hbeta_j|/\lam), & \hbeta_j\neq 0, \\
\bx_j'(\by - \bX\hbbeta)/n \in \lam[-1,1], & \hbeta_j=0.
\end{cases}
\eel
If the penalized loss (\ref{pen-loss-beta}) is convex in $\bbeta$, then (\ref{KKT}) is the Karush--Kuhn--Tucker
condition for its minimization.

Given a penalty function $\rho(\cdot)$, one still has to choose a penalty level $\lam$ to arrive at a solution
of (\ref{KKT}). Such a choice may depend on the purpose of estimation, since variable selection may require a larger $\lam$ than does prediction. However, scale-invariance considerations and theoretical results suggest using a penalty level proportional to the noise level $\sigma$. This motivates a scaled penalized least squares estimator as a numerical equilibrium in the following iterative algorithm:
\bel{gen-alg}
\begin{split}
&\hsigma \leftarrow {|}\by-\bX\hbbeta^{\mathrm{old}}{|}_2/\{(1-a)n\}^{1/2}, \\
&\lam \leftarrow \hsigma\lam_0, \\
&\hbbeta \leftarrow \hbbeta^{\mathrm{new}},\ L_\lam(\hbeta^{\mathrm{new}})
\le  L_\lam(\hbeta^{\mathrm{old}}),
\end{split}
\eel
where $\lam_0$ is a prefixed penalty level, not depending on $\sigma$, $\hsigma$ estimates the noise level, and $a\ge 0$ provides an option for a degrees-of-freedom adjustment with $a>0$. For $p<n$ and $(a,\lam_0)=(p/n,0)$, (\ref{gen-alg}) initialized with the least squares estimator
$\hbbeta^{\mathrm{(lse)}}$ is non-iterative and gives
$\hsigma^2={|}\by-\bX\hbbeta^{\mathrm{(lse)}}{|}_2^2/(n-p)$.
For large data sets, one may use a few passes of a gradient descent algorithm
to compute $\hbeta^{\mathrm{new}}$ from $\hbeta^{\mathrm{old}}$.
For $a=0$, this algorithm was considered in \citet{SunZ10}.
In \citet{SunZ10} and the numerical experiments reported in Section 4,
$\hbeta^{\mathrm{new}}$ is a solution of (\ref{KKT}) for the given $\lam$.
We describe this implementation in the following two paragraphs.

The first step of our implementation is the computation of a solution path
$\hbeta(\lam)$ of (\ref{KKT}) beginning from $\hbeta(\lam)=0$ for $\lam = |X'y/n|_\infty$.
For quadratic spline penalties $\rho(t)$ with $m$ knots, \citet{Zhang10-mc+} developed an algorithm to compute a linear spline path of solutions $\{\lam^{(t)}\oplus \hbbeta^{(t)}: t\ge 0\}$ of (\ref{KKT}) to cover the entire range of $\lam$. This extends the least angle regression solution or lasso path \citep{OsbornePT00a,OsbornePT00b,EfronHJT04} from $m=1$ and includes the minimax concave penalty
for $m=2$ and the smoothly clipped absolute deviation penalty \citep{FanL01} for $m=3$.
An R package named plus is available for computing the solution paths for these penalties.

The second step of our implementation is the iteration (\ref{gen-alg}) along the solution path
$\beta(\lam)$ computed in the first step. That is to use the already computed
\bel{alg}
\hbeta^{\mathrm{new}} = \hbbeta(\lam)
\eel
in (\ref{gen-alg}).
For the {scaled} lasso, we use $a=0$ in (\ref{gen-alg}) and $\rho(t)=t$ in (\ref{pen-loss-beta}) and (\ref{KKT}). For the {scaled} minimax concave penalized selection, we use $a=0$ and the minimax concave penalty $\rho(t) = \int_0^t(1-x/\gamma)_+{\rm d}x$, where $\gamma>0$ regularizes the maximum concavity of the penalty. When $\gamma=\infty$, it becomes the {scaled} lasso. The algorithm (\ref{gen-alg}) can be easily implemented once a solution path is computed.

Consider the $\ell_1$ penalty. As discussed in the introduction, (\ref{gen-alg}) and (\ref{alg})
form an alternating minimization algorithm for the penalized joint loss function
\bel{pen-loss-joint}
L_{\lam_0}(\bbeta,\sigma) =
\frac{{|}\by-\bX\bbeta{|}_2^2}{2n\sigma} + \frac{(1-a)\sigma}{2} +
\lam_0{|}\bbeta{|}_1.
\eel
\citet{Antoniadis10} suggested this jointly convex loss function as a way of extending Huber's robust regression method to high dimensions.
For $a=0$ and $\lam=\hsigma\lam_0$ with fixed $\hsigma$,
$\hsigma L_{\lam_0}(\bbeta,\hsigma) = L_\lam(\bbeta)+\hsigma^2/2$,
so that $\hbbeta\leftarrow \hbbeta(\lam)$ in (\ref{alg}) minimizes
$L_{\lam_0}(\bbeta,\hsigma)$ over $\bbeta$. For fixed $\hbbeta$,
$\hsigma^2 \leftarrow {|}\by-\bX\hbbeta{|}_2^2/\{(1-a)n\}$ in (\ref{gen-alg})
minimizes $L_{\lam_0}(\hbbeta,\sigma)$ over $\sigma$.
During the revision of this paper, we learned that \cite{SheO11} have considered penalizing Huber's concomitant loss function for outlier detection in linear regression.
We summarize some properties of the algorithm (\ref{gen-alg}) with (\ref{alg}) in the following proposition.

\begin{proposition}\label{prop-1} Let $\hbbeta=\hbbeta(\lam)$ be a solution path of (\ref{KKT}) with $\rho(t)=t$.
The penalized loss function (\ref{pen-loss-joint}) is jointly convex in $(\bbeta, \sigma)$ and the algorithm
(\ref{gen-alg}) with (\ref{alg}) converges to
\bel{estimator}
(\hbbeta,\hsigma) = \argmin_{\bbeta,\sigma}L_{\lam_0}(\bbeta,\sigma).
\eel
The resulting estimators $\hbbeta=\hbbeta(\bX,\by)$ and $\hsigma=\hsigma(\bX,\by)$ are scale equivariant in $\by$ in the sense that
$\hbbeta(\bX,c\by) = c\hbbeta(\bX,\by)$ and $\hsigma(\bX,c\by)=|c|\hsigma(\bX,\by)$. Moreover,
\bel{prop-1-1}
\frac{\pa}{\pa\sigma} L_{\lam_0}\big\{\hbbeta(\sigma\lam_0),\sigma\big\}
= \frac{1-a}{2} - \frac{{|}\by-\bX\hbbeta(\sigma\lam_0){|}_2^2}{2n\sigma^2}.
\eel
\end{proposition}
Since (\ref{pen-loss-joint}) is not strictly convex, the joint estimator may not be unique for some data $(X,y)$. However, since (\ref{pen-loss-joint}) is strictly convex in $\sigma$, $\hsigma$ is always unique in (\ref{estimator}) and the uniqueness of $\hbeta$ follows from that of the lasso estimator $\hbeta(\lam)$ at $\lam=\hsigma\lam_0$; $\hbbeta(\lam)$ is unique when the second part of (\ref{KKT}) is strict in the sense of not hitting $\pm \lam$ when $\hbeta_j=0$, which holds almost everywhere in $(X,y)$ for $\lam>0$. See, for example, \cite{Zhang10-mc+}.

Let $\hsigma(\lam) = |y-X\hbeta(\lam)|_2/\{(1-a)n\}^{1/2}$. For $\lam_0=\{(2/n)\log p\}^{1/2}$,
(\ref{prop-1-1}) implies that
\bel{zhang-10-format}
\hsigma =\hsigma(\hlam),\  \hlam = \min\big\{\lam: \hsigma^2(\lam) \le n\lam^2/(2\log p)\big\}.
\eel
While the present paper continues our earlier work \citep{SunZ10} by providing further theoretical
and numerical justifications for (\ref{gen-alg}) and (\ref{alg}), the estimator has appeared in different forms.
In addition to (\ref{pen-loss-joint}) and (\ref{estimator}) of \cite{Antoniadis10}, (\ref{zhang-10-format}) appeared in
\citet{Zhang10-mc+}.
While this paper was in revision, a reviewer called our attention to \cite{BelloniCW11}, who focused on
studying $\hbeta$ in an equivalent form as square-root lasso.
We note that (\ref{gen-alg}) and (\ref{alg}) allow concave penalties and degrees of freedom
adjustments as in \citet{Zhang10-mc+}.

\section{Theoretical results}
\subsection{Analysis of scaled lasso}
Let $\bbeta^*$ be a vector of true regression coefficients.
An expert with oracular knowledge of $\bbeta^*$ would estimate the noise level by the oracle estimator
\bel{sigma^*}
\sigma^*={|}\by-\bX\bbeta^*{|}_2/n^{1/2}.
\eel
Under the Gaussian assumption, this is the maximum likelihood estimator for $\sigma$ when $\beta^*$ is known and $n(\sigma^*/\sigma)^2$ follows the $\chi^2_n$ distribution.
Due to the scale equivariance of $\hsigma$ in Proposition 1, it is natural to use $\sigma^*$ as an estimation target with or without the Gaussian assumption.
We derive upper and lower bounds for $\hsigma/\sigma^*-1$ and use them to prove the consistency and asymptotic normality of $\hsigma$. We derive oracle inequalities for the prediction performance and the estimation of $\bbeta$ under the $\ell_q$ loss. Throughout the sequel, $\mathrm{pr}_{\bbeta,\sigma}$ is the probability measure under which $\by-\bX\bbeta\sim N(0, \sigma^2\bI_n)$. We assume that ${|}\bx_j{|}_2^2=n$ whenever $\mathrm{pr}_{\bbeta,\sigma}$ is invoked. The asymptotic theory here concerns $n\to\infty$ and
allows all parameters and variables to depend on $n$, including $p\ge n\ge |\beta|_0\to \infty$.

We first provide the consistency for the estimation of $\sigma$ via an oracle inequality for the
prediction error of the scaled lasso.  In our first theorem, the relative error for the estimation
of $\sigma$ is bounded by a quantity $\tau_0$ related to a prediction error bound
$\eta(\lam,\xi,w,T)$ in (\ref{eta}) below.
For $\lam>0$, $\xi>1$, $\bw\in\real^p$, and $T\subset\{1,\dots,p\}$, define $\delta_{w,T}= 1 - I(w=\beta^*,T=\emptyset)$ and
\bel{eta}
\eta(\lam,\xi,\bw,T) = {|}\bX\bbeta^*-\bX\bw{|}_2^2/n
+(1+\delta_{w,T})2\lam{|}\bw_{T^c}{|}_1+ \frac{4\xi^2\lam^2|T|}{(\xi+1)^2\kappa^2(\xi,T)}
\eel
where $\kappa(\xi,T)$, the compatibility factor \citep{vandeGeerB09}, is defined as
\bel{compatible}
\kappa(\xi,T) = \min\Big\{\frac{|T|^{1/2}{|}\bX\bu{|}_2}{n^{1/2}{|}\bu_T{|}_1}:
\bu\in \scrC(\xi,T),\ \bu\neq 0\Big\}
\eel
with the cone $\scrC(\xi,T)=\{\bu: {|}\bu_{T^c}{|}_1\le\xi{|}\bu_T{|}_1\}$.
Since the prediction error bound $\eta(\lam,\xi,w,T)$ is valid for all $w$ and $T$,
$\tau_0$ is related to its minimum over all $w$ and $T$ at the oracle scale $\sigma^*$:
\bel{min-eta}
\tau_0=\eta_*^{1/2}(\sigma^*\lam_0,\xi)/\sigma^*,\quad
\eta_*(\lam,\xi)=\inf_{\bw,T}\ \eta(\lam,\xi,\bw,T).
\eel

\begin{theorem}\label{th-1}
Let $(\hbbeta,\hsigma)$ be the scaled lasso estimator in (\ref{estimator}),
$\bbeta^*\in \real^p$, $\sigma^*$ the oracle noise level in (\ref{sigma^*}),
$z^*={|}\bX'(\by-\bX\bbeta^*)/n{|}_\infty/\sigma^*$ and $\xi>1$.
When $z^* \le  (1-\tau_0)\lam_0(\xi-1)/(\xi+1)$,
\bel{th-1-1}
\max\Big(1-\frac{\hsigma}{\sigma^*}, 1-\frac{\sigma^*}{\hsigma}\Big) \le \tau_0,\quad
\frac{{|}\bX\hbbeta-\bX\bbeta^*{|}_2}{n^{1/2}\sigma^*}
\le \frac{1}{\sigma^*}\eta_*^{1/2}\Big(\frac{\sigma^*\lam_0}{1-\tau_0},\xi\Big)
\le \frac{\tau_0}{1-\tau_0}.
\eel
In particular, if $\lam_0= A\{(2/n)\log p\}^{1/2}$ with $A>(\xi+1)/(\xi-1)$ and $\eta_*(\sigma\lam_0,\xi)/\sigma\to 0$, then
\bel{th-1-2}
\mathrm{pr}_{\bbeta^*,\sigma}\big(|\hsigma/\sigma-1|>\eps\big)\to 0
\eel
for all $\eps>0$.
\end{theorem}

\medskip
Theorem \ref{th-1} extends to the scaled lasso a unification of prediction oracle
inequalities for a fixed penalty. With $\lam=\sigma^*\lam_0/(1-\tau_0)_+$, (\ref{th-1-1}) gives
$\max\{(\sigma^*\tau_0)^2,{|}\bX\hbbeta-\bX\bbeta^*{|}_2^2/n\}\le \eta_*(\lam,\xi)$, or
\bel{th-1-3}
\max\{(\sigma^*\tau_0)^2,{|}\bX\hbbeta-\bX\bbeta^*{|}_2^2/n\}
\le\min_w \Big\{{|}\bX w-\bX\bbeta^*{|}_2^2/n + 4\Ctil\lam\sum_{j=1}^p\min(\lam,|w_j|)\Big\}
\eel
for a $\Ctil\ge 1$,
if the minimum in (\ref{th-1-3}) is attained at a $\wtil$ with $(1+1/\xi)^2\kappa^2(\xi,\Ttil)\ge 1/\Ctil$,
where $\Ttil=\{j: |\wtil_j|>\lam\}$. This asserts that for an arbitrary, possibly non-sparse
$\beta^*$, the prediction error of the scaled lasso is no greater
than that of the best linear predictor $Xw$ with a sparse $w$ for an additional capped-$\ell_1$
cost of the order $\lam\sum_j \min(\lam,|w_j|)$.
A consequence of this prediction error bound for the scaled lasso is the consistency of the corresponding estimator of the noise level in (\ref{th-1-2}). Due to the scale equivariance in Proposition \ref{prop-1}, Theorem \ref{th-1} and the results in the rest of the section are all scale free.

For fixed penalty $\lam$, the upper bound $\eta(\lam,\xi,w,T)$
has been previously established for different $w$ and $T$,
with possibly different constant factors. Examples include
$\eta(\lam,\xi,\beta^*,\emptyset)=2\lam|\beta^*|_1$ \citep{GreenshteinR04, Greenshtein06},
$\eta(\lam,\xi,\beta^*,S_{\beta^*})\lesssim \lam^2|\beta^*|_0$ with
$S_w=\{j: w_j\neq 0\}$ \citep{vandeGeerB09}, and $\min_w\eta(\lam,\xi,w,S_w)
= \min_w\{|X\beta^*-Xw|_2^2/n+O(\lam^2|w|_0)\}$ \citep{KoltchinskiiLT10}.
In (\ref{eta}), the coefficient for $|Xw-X\beta^*|_2^2/n$ is 1 as in \cite{KoltchinskiiLT10}.

Now we provide sharper convergence rates and the asymptotic normality for the scaled lasso
estimation of the noise level $\sigma$. This sharper rate $\lam\mu(\lam,\xi)/\sigma^2$,
essentially taking the square of the order $\tau_0$ in (\ref{th-1-1}),
is based on the following $\ell_1$ error bound for the estimation of $\bbeta$,
\bel{mu}
\mu(\lam,\xi) = (\xi+1)\min_{T}\inf_{0<\nu<1}
\max\Big[\frac{{|}\bbeta^*_{T^c}{|}_1}{\nu} ,\frac{\lam|T|/\{2(1-\nu)\}}
{\kappa^2\{(\xi+\nu )/(1-\nu ),T\}}\Big].
\eel
This $\ell_1$ error bound has the interpretation
\bel{ell_1-rate}
|\hbeta - \beta^*|_1\le \mu(\lam,\xi) \le \Ctil\sum_{j=1}^p\min(\lam,|\beta^*_j|),
\eel
if $\Ctil\ge (1+\xi)\max\{2,1/\kappa^2(2\xi+1,\Ttil)\}$ with
$\Ttil=\{j: |\beta_j^*|>\lam\}$.
This allows $\beta^*$ to have many small elements, as in
\citet{ZhangH08}, \citet{Zhang09-l1} and \cite{YeZ10}.
The bound $\mu(\lam,\xi)\le (\xi+1)\lam|S_{\beta^*}|/\{2\kappa^2(\xi,S_{\beta^*})\}$ improves
upon its earlier version in \citet{vandeGeerB09} by a constant factor $4\xi/(\xi+1)\in (2,4)$.

\begin{theorem}\label{th-2} Let $\hbbeta,\hsigma, \bbeta^*,\sigma^*,z^*$ and $\xi$ be as in
Theorem \ref{th-1}. Set $\tau_*=\{\lam_0\mu(\sigma^*\lam_0,\xi)/\sigma^*\}^{1/2}$.
(i) The following
inequalities hold when $z^* \le  (1-\tau_*^2)\lam_0(\xi-1)/(\xi+1)$,
\bel{th-2-1}
\max\big(1- \hsigma/\sigma^*,1-\sigma^*/\hsigma)\le \tau_*^2,\quad
{|}\hbbeta-\bbeta^*{|}_1 \le \mu(\sigma^*\lam_0,\xi)/(1-\tau_*^2).
\eel
(ii) Let $\lam_0\ge \{(2/n)\log(p/\eps)\}^{1/2}(\xi+1)/\{(\xi-1)(1-\tau_*^2)\}$.
For all $\eps>0$ and $n-2 > \log(p/\eps) \to\infty$,
\bes
\mathrm{pr}_{\bbeta^*,\sigma}\big\{z^* \le  (1-\tau_*^2)\lam_0(\xi-1)/(\xi+1)\big\}
\ge 1-\{1+o(1)\}\eps/\{\pi\log(p/\eps)\}^{1/2}.
\ees
If $\lam_0= A\{(2/n)\log p\}^{1/2}$ with $A>(\xi+1)/(\xi-1)$
and $\lam_0\mu(\sigma\lam_0,\xi)/\sigma \ll n^{-1/2}$, then
\bel{th-2-2}
n^{1/2}\big(\hsigma/\sigma-1\big)\to N(0,1/2)
\eel
in distribution under $\mathrm{pr}_{\bbeta^*,\sigma}$.
\end{theorem}

Since $\sigma^2 \tau_*^2\approx \mu(\lam,\xi) \le 2(\xi+1)\min_T \eta(\lam,2\xi+1,\bbeta^*,T)$ with
$\lam = \sigma\lam_0$, the rate $\tau_*^2$ in (\ref{th-2-1}) is essentially the square of that in (\ref{th-1-1}),
in view of (\ref{min-eta}). It follows that the scaled lasso provides a faster convergence rate
than does the penalized maximum likelihood estimator for the estimation of the noise level
\citep{StadlerBG10,SunZ10}. In particular, (\ref{th-2-1}) implies that
\bel{th-2-3}
\max\big(1- \hsigma/\sigma^*,1-\sigma^*/\hsigma)
\le (\xi+1)\lam_0^2|S_{\beta^*}|/\{2\kappa^2(\xi,S_{\beta^*})\}
\lesssim |\beta^*|_0(\log p)/n
\eel
with $S_{\beta^*}=\{j:\beta^*_j\neq 0\}$,
when $\kappa^2(\xi,S_{\beta^*})$ can be treated as a constant.
The bounds in (\ref{th-2-3}) and its general version (\ref{th-2-1}) lead to the asymptotic normality (\ref{th-2-2}) under proper
assumptions. Thus, statistical inference
about $\sigma$ is justified with the scaled lasso in certain large-$p$-smaller-$n$ cases, for example,
when $|\beta^*|_0(\log p)/\surd n\to 0$ under the compatibility condition \citep{vandeGeerB09}.

For a fixed penalty level, oracle inequalities for the $\ell_q$ error of the lasso have been established in
\citet{BuneaTW07}, \citet{vandeGeer08} and \citet{vandeGeerB09} for $q=1$,
\citet{ZhangH08} and \citet{BickelRT09} for $q\in [1,2]$,
\citet{MeinshausenY09} for $q=2$, and \citet{Zhang09-l1} and \citet{YeZ10} for $q\ge 1$.
The bounds on $\hsigma/\sigma^*$ in (\ref{th-2-1}) and (\ref{th-2-3}) allow automatic extensions of these existing $\ell_q$ oracle inequalities from the lasso with fixed penalty to the scaled lasso. We illustrate this by extending the oracle inequalities of \citet{YeZ10} for the lasso and \citet{CandesT07} for the Dantzig selector in the following corollary.
\citet{YeZ10} used the following sign-restricted cone invertibility factor to separate
conditions on the error $y-X\beta^*$ and design $X$ in the derivation of error bounds for the lasso:
\bel{scif}
F_q(\xi,S)
=\inf\Big\{\frac{|S|^{1/q}{|}\bX'\bX\bu{|}_\infty}{n{|}\bu{|}_q}:\bu\in\scrC_-(\xi,S)\Big\},
\eel
where $\scrC_-(\xi,S)=\{\bu: {|}\bu_{S^c}{|}_1\le\xi{|}\bu_S{|}_1\neq 0, u_j\bx_j'\bX\bu\le0, \text{for all } j\not\in S\}$.
The quantity (\ref{scif}) can be viewed as a generalized restricted eigenvalue comparing the $\ell_q$
loss and the dual norm of the $\ell_1$ penalty with respect to the inner product for the least squares fit.
This gives a direct connection to the Karush--Kuhn--Tucker condition (\ref{KKT}).
Compared with the restricted eigenvalue \citep{BickelRT09} and
the compatibility factor (\ref{compatible}), a main advantage of (\ref{scif}) is to allow all $q\in [1,\infty]$.
In addition, (\ref{scif}) yields sharper oracle inequalities \citep{YeZ10}.
For $(|A|,|B|,{|}\bu{|}_2)=(\lceil a\rceil,\lceil b\rceil,1)$ with $A\cap B=\emptyset$, define
\bel{SRC}
\delta^\pm_a = \max_{A,\bu}\Big\{\pm\Big({|}\bX_A\bu/n^{1/2}{|}_2-1\Big)\Big\},\
\theta_{a,b} = \max_{A,B,\bu}\big{|}\bX_A'\bX_B\bu/n\big{|}_2.
\eel
The quantities in (\ref{SRC}) are used in the uniform uncertainty principle \citep{CandesT07} and
the sparse Riesz condition \citep{ZhangH08}. We note that $1-\delta_a^-$ is the minimum eigenvalue
of $X_A'X_A/n$ among $|A|\le a$, $1+\delta_a^+$ is the corresponding maximum eigenvalue,
and $\theta_{a,b}$ is the maximum operator norm of size $a\times b$
off-diagonal sub-blocks of the Gram matrix $X'X/n$.

\begin{corollary}\label{cor1} Suppose ${|}\bbeta^*_{S^c}{|}_1=0$. Then, Theorem \ref{th-2} holds with $\mu(\lam,\xi)$ replaced by $\lam|S|(2\xi)/\{(\xi+1)F_1(\xi,S)\}$, and for $z^* \le  (1-\tau_*^2)\lam_0(\xi-1)/(\xi+1)$,
\bel{equ-thm1-3}&&{|}\hbbeta-\bbeta^*{|}_q
\le \frac{k^{1/q}(\sigma^* z^*+\hsigma\lam_0)}{F_q(\xi,S)}
\le \frac{2\sigma^*\xi\lam_0k^{1/q}}{(1-\tau_*^2)(\xi+1)F_q(\xi,S)}
\eel
for all $1\le q\le \infty$, where $k=|S|$.
In particular, for $\xi=\surd 2$ and $z^* \le  (1-\tau_*^2)\lam_0(\surd{2}-1)^2$,
\bel{eigen}
{|}\hbbeta-\bbeta^*{|}_2
\le\frac{(8k)^{1/2}\lam_0\sigma^*/(1-\tau_*^2)}{(\surd{2}+1)F_2(\surd{2},S)}
\le\frac{4k^{1/2}\lam_0\sigma^*/(1-\tau_*^2)}
{(\surd{2}+1)(1-\delta^-_{1.5k}-\theta_{2k,1.5k})_+}.
\eel
\end{corollary}

The proofs of Theorems 1 and 2 are based on a basic inequality
\bel{basic}
&& \hspace{-.8in} |X\hbeta(\lam)-X\beta^*|_2^2/n + |X\hbeta(\lam)-X w|_2^2/n
\cr&& \le |X w - X\beta^*|_2^2/n + 2\lam\{{|}\bw{|}_1-{|}\hbbeta(\lam){|}_1\}
+ 2\sigma^* z^*{|}\bw-\hbbeta(\lam){|}_1 \qquad\quad
\eel
as a consequence of the Karush--Kuhn--Tucker conditions (\ref{KKT}).
The version of (\ref{basic}) with $\bw=\bbeta^*$ is well-known \citep{vandeGeerB09}
and controls $|X\hbeta(\lam)-X\beta^*|_2^2$ for sparse $\beta^*$.
When $|X\hbeta(\lam)-X\beta^*|_2^2 > |Xw-X\beta^*|_2^2$,
(\ref{basic}) controls the excess for sparse $w$ by the same argument.
The general $\bw$ is taken in Theorem 1, while $\bw=\bbeta^*$ is taken in Theorem 2.
In both cases, (\ref{basic}) provides the cone condition in (\ref{compatible}) and (\ref{scif}).
This is used to derive upper and lower bounds for (\ref{prop-1-1}),
the derivative of the profile loss function $L_{\lam_0}(\hbbeta(\sigma\lam_0),\sigma)$
with respect to $\sigma$, within a neighborhood of $\sigma/\sigma^*=1$.
The bounds for the minimizer $\hsigma$ then follow from the joint convexity
of the penalized loss (\ref{pen-loss-joint}).

\subsection{Estimation after model selection}
We have proved that without requiring the knowledge of $\sigma$, the scaled lasso enjoys prediction
and estimation properties comparable to the best known theoretical results for known $\sigma$,
and the scaled lasso estimate of $\sigma$ enjoys consistency and asymptotic normality properties under
proper conditions.
However, the lasso estimator may have substantial bias \citep{FanP04, Zhang10-mc+}, and its bias is
significant in our own simulation experiments.
Although the smoothly clipped absolute deviation and minimax concave penalized selectors
were introduced to remove the bias of the lasso \citep{FanL01,Zhang10-mc+}, a theoretical study of
their scaled version (\ref{gen-alg}) is beyond the scope of this paper. In this subsection, we present theoretical
results for another bias removing method: least squares estimation after model selection.

Given an estimator $\hbeta$ of the coefficient vector $\beta$,
the least squares estimator of $\beta$ and the corresponding estimator of the noise level
$\sigma$ in the model selected by $\hbeta$ are
\bel{mleas}
\hhbeta = \argmin_{\beta}\Big\{|y-X\beta|_2^2:\supp(\beta)\subseteq\supp(\hbeta)\Big\},\quad
\hhsigma = \big|y-X\hhbeta\,\big|_2\big/\surd n,
\eel
where $\supp(\beta)=\{j: \beta_j\neq 0\}$. Alternatively, we may use
$\hhsigma = \big|y-X\hhbeta\,\big|_2\big/\surd(n-|\hbeta|_0)$ to estimate the noise level.
However, since the effect of this degrees of freedom adjustment is of smaller order than
our error bound, we will focus on the simpler (\ref{mleas}).

In addition to the compatibility factor $\kappa(\xi,S)$ in (\ref{compatible}), we use sparse eigenvalues
to study the least squares estimation after the scaled lasso selection.
Let $\lam_{\min}(M)$ be the smallest eigenvalue of a matrix $M$ and $\lam_{max}(M)$ the largest .
For models $T\subset\{1,\ldots,p\}$, define
\bes
\kappa_-(m,T) = \min_{B\supset T, |B\setminus T|\le m}\lam_{\min}(X_B'X_B/n),\
\kappa_+(m,T) = \min_{B\cap T\emptyset, |B|\le m}\lam_{\min}(X_B'X_B/n),
\ees
as the sparse lower eigenvalue of the Gram matrix for models containing $T$
and the sparse upper eigenvalue for models disjoint with $T$.
Let $S = \supp(\beta^*)$ and ${\widehat S}=\supp(\hbeta)$.
The following theorem provides prediction and estimation error bounds for (\ref{mleas}) after the scaled lasso selection,
along with an upper bound for the false positive $|{\widehat S}\setminus S|$, a key element in our study.

\begin{theorem}\label{th-mleas}
Let  $(\hbbeta,\hsigma)$ be the scaled lasso estimator in (\ref{estimator}) and
$(\hhbeta,\hhsigma)$ the least squares estimator (\ref{mleas}) in model ${\widehat S}$.
Let $\beta^*,\sigma^*,z^*,\xi$ and $\tau_*$ be as in Theorem \ref{th-2} and $m$ be an integer satisfying $|S|\xi^2/\kappa^2(\xi,S)<m/\kappa_+(m,S)$.
If $z^* \le  (1-\tau_*^2)\lam_0(\xi-1)/(\xi+1)$, then
\bel{th-mleas-1}
|{\widehat S}\setminus S| < m,\
\hsigma^2 - \{ \sigma^*_{m-1,S} + \surd \eta_*(\hlam,\xi)\}^2 \le \hhsigma^2 \le\hsigma^2,
\eel
with $\hlam = \hsigma \lam_0 \le \sigma^*\lam_0/(1-\tau_*^2)$
and $\sigma^*_{m,S} =  \max_{B\supseteq S, |B\setminus S| = m}|(y-X\beta^*)_B|_2/\surd n$,
and
\bel{th-mleas-2}
\kappa_-(m-1,S)|\hhbeta - \beta^*|_2^2 \le |X\hhbeta-X\beta^*|_2^2/n
\le \big\{\sigma^*_{m-1,S} + 2 \surd \eta_*(\hlam,\xi)\big\}^2.
\eel
Moreover, in addition to the probability bound for $z^* \le  (1-\tau_*^2)\lam_0(\xi-1)/(\xi+1)$ in
Theorem~\ref{th-2} (ii), for all integers $1\le m\le p$,
\bel{th-mleas-3}
\mathrm{pr}_{\bbeta^*,\sigma}\big[ \sigma^*_{m,S}(\surd n)/\sigma
\ge \surd(m+|S|) + \surd\{(2m)\log(ep/m)+2\log(1/\eps)\} \big] \le \frac{\eps/m}{\surd(2\pi)}.
\eel
\end{theorem}

For Gaussian design matrices, the sparse eigenvalues $\kappa_-(m,\emptyset)$
and $\kappa_+(m,\emptyset)$ can be treated as constants when $m(\log p)/n$ is small
and the eigenvalues of the expected Gram matrix are uniformly bounded away from zero and
infinity \citep{ZhangH08}. Since  $\kappa_-(m,S)\ge \kappa_-(m+|S|,\emptyset)$ and
$\kappa_+(m,S)\le \kappa_+(m,\emptyset)$, they can be treated as constants
in the same sense in Theorem \ref{th-mleas}.
Thus, for sufficiently small $|S|(\log p)/n$, we may take an $m$ of the
same order as $|S|$. In this case, the difference between $(\hhbeta,\hhsigma)$ and the scaled lasso
estimator $(\hbeta,\hsigma)$ is of no greater order than the difference between
$(\hbeta,\hsigma)$ and the estimation target $(\beta^*,\sigma^*)$.
Consequently, 
\bes
\big|\hhsigma/\sigma -1\big| + \big|\hbbeta - \beta^*\big|_2^2 + \big|X\hhbeta-X\beta^*\big|_2^2/n
= O_P(1)|S|(\log p)/n.
\ees

As we have mentioned earlier, the key element in our analysis of (\ref{mleas}) is the bound
$|{\widehat S}\setminus S| < m$ in (\ref{th-mleas-1}). Since this is a weaker assertion
than variable consistency ${\widehat S}=S$, the conditions of Theorem \ref{th-mleas} on the design
matrix is of a weaker form than the irrepresentability condition for variable selection consistency
\citep{MeinshausenB06, ZhaoY06}. In \cite{ZhangH08} and \cite{Zhang10-mc+}, upper bounds
for the false positive were obtained under a sparse Riesz condition
on $\kappa_-(m,\emptyset)$ and $\kappa_+(m,\emptyset)$.

\begin{table}
\def~{\hphantom{0}}
\captionsetup{width=4.5in}
\tbl{Performance of five methods in Example 1 at penalty levels $\lam_j=\{2^{j-1}(\log p) /n\}^{1/2}$
$(j=1,2,3)$, across 100 replications, in terms of the bias ($\times10$) and standard error ($\times10$)
of $\hsigma/\sigma$ for the selector and $\hhsigma/\sigma$ for the least squares estimator after model selection,
the average model size, and the relative frequency of sure screening,
along with the simulation results in \citet{FanGH10}}{
\begin{tabular}{c c cccc cccc}\\
&&\multicolumn{4}{c}{${r_0}=0$}&\multicolumn{4}{c}{${r_0}=$\,0$\cdot$5}\\
Method & &\multicolumn{1}{c}{$\hsigma/\sigma$} & $\hhsigma/\sigma$ & AMS & SSP
         &\multicolumn{1}{c}{$\hsigma/\sigma$} & $\hhsigma/\sigma$ & AMS & SSP  \\\\
        &$\lam_1$ & 1$\cdot$6$\pm$0$\cdot$6 & $-$1$\cdot$1$\pm$0$\cdot$7 & 7$\cdot$6 & 1$\cdot$0 & 1$\cdot$5$\pm$0$\cdot$6 & $-$1$\cdot$0$\pm$0$\cdot$7 & 9$\cdot$7 & 1$\cdot$0\\
PMLE    &$\lam_2$ & 2$\cdot$5$\pm$0$\cdot$6 & $-$0$\cdot$1$\pm$0$\cdot$6 & 3$\cdot$0 & 1$\cdot$0 & 2$\cdot$5$\pm$0$\cdot$6 & $-$0$\cdot$3$\pm$0$\cdot$6 & 5$\cdot$2 & 1$\cdot$0 \\
        &$\lam_3$ & 3$\cdot$6$\pm$0$\cdot$7 & 1$\cdot$2$\pm$1$\cdot$1 & 1$\cdot$8 & 0$\cdot$2 & 3$\cdot$8$\pm$0$\cdot$6 & $-$0$\cdot$2$\pm$0$\cdot$6 & 3$\cdot$7 & 1$\cdot$0 \\\\
        &$\lam_1$ & 0$\cdot$5$\pm$0$\cdot$6 & $-$1$\cdot$8$\pm$0$\cdot$7 & 11$\cdot$9 & 1$\cdot$0 & 0$\cdot$0$\pm$0$\cdot$6 & $-$1$\cdot$9$\pm$0$\cdot$7 & 15$\cdot$5 & 1$\cdot$0 \\
BC      &$\lam_2$ & 1$\cdot$6$\pm$0$\cdot$6 & $-$0$\cdot$1$\pm$0$\cdot$6 & 3$\cdot$1 & 1$\cdot$0 & 0$\cdot$7$\pm$0$\cdot$6 & $-$0$\cdot$3$\pm$0$\cdot$6 & 6$\cdot$1 & 1$\cdot$0 \\
        &$\lam_3$ & 3$\cdot$3$\pm$0$\cdot$7 & 1$\cdot$1$\pm$1$\cdot$1 & 1$\cdot$9 & 0$\cdot$3 & 1$\cdot$9$\pm$0$\cdot$7 & $-$0$\cdot$2$\pm$0$\cdot$6 & 4$\cdot$2 & 1$\cdot$0\\\\
Scaled  &$\lam_1$  & 0$\cdot$0$\pm$0$\cdot$6 & $-$2$\cdot$1$\pm$0$\cdot$8 & 14$\cdot$6 & 1$\cdot$0 & $-$0$\cdot$5$\pm$0$\cdot$6 & $-$2$\cdot$3$\pm$0$\cdot$7 & 18$\cdot$6 & 1$\cdot$0 \\
lasso   &$\lam_2$ & 1$\cdot$3$\pm$0$\cdot$7 & $-$0$\cdot$2$\pm$0$\cdot$6 & 3$\cdot$1 & 1$\cdot$0 & 0$\cdot$4$\pm$0$\cdot$6 & $-$0$\cdot$3$\pm$0$\cdot$6 & 6$\cdot$2 & 1$\cdot$0 \\
        &$\lam_3$ & 3$\cdot$1$\pm$0$\cdot$7 & 1$\cdot$0$\pm$1$\cdot$1 & 1$\cdot$9 & 0$\cdot$3 & 1$\cdot$2$\pm$0$\cdot$7 & $-$0$\cdot$2$\pm$0$\cdot$6 & 4$\cdot$4 & 1$\cdot$0 \\\\
Scaled  &$\lam_1$  & $-$1$\cdot$2$\pm$0$\cdot$8 & $-$2$\cdot$4$\pm$0$\cdot$8 & 14$\cdot$1 & 1$\cdot$0 & $-$0$\cdot$7$\pm$0$\cdot$6 & $-$2$\cdot$2$\pm$0$\cdot$8 & 13$\cdot$8 & 1$\cdot$0 \\
MCP     &$\lam_2$ & $-$0$\cdot$1$\pm$0$\cdot$6 & $-$0$\cdot$1$\pm$0$\cdot$6 & 3$\cdot$1 & 1$\cdot$0 & 0$\cdot$1$\pm$0$\cdot$6 & $-$0$\cdot$2$\pm$0$\cdot$6 & 3$\cdot$2 & 1$\cdot$0 \\
        &$\lam_3$ & 1$\cdot$5$\pm$1$\cdot$3 & 0$\cdot$6$\pm$1$\cdot$1 & 2$\cdot$4 & 0$\cdot$6 & 0$\cdot$6$\pm$0$\cdot$7 & $-$0$\cdot$1$\pm$0$\cdot$6 & 3$\cdot$0 & 1$\cdot$0 \\\\
Scaled  &$\lam_1$ & $-$0$\cdot$6$\pm$0$\cdot$6 & $-$2$\cdot$2$\pm$0$\cdot$7 & 14$\cdot$0 & 1$\cdot$0 & $-$0$\cdot$4$\pm$0$\cdot$6 & $-$2$\cdot$2$\pm$0$\cdot$8 & 13$\cdot$9 & 1$\cdot$0 \\
SCAD    &$\lam_2$ & 0$\cdot$8$\pm$1$\cdot$0 & $-$0$\cdot$1$\pm$0$\cdot$6 & 3$\cdot$1 & 1$\cdot$0 & 0$\cdot$4$\pm$0$\cdot$6 & $-$0$\cdot$3$\pm$0$\cdot$6 & 3$\cdot$8 & 1$\cdot$0 \\
        &$\lam_3$ & 3$\cdot$1$\pm$0$\cdot$7 & 0$\cdot$9$\pm$1$\cdot$1 & 2$\cdot$0 & 0$\cdot$3 & 1$\cdot$2$\pm$0$\cdot$7 & $-$0$\cdot$2$\pm$0$\cdot$6 & 3$\cdot$8 & 1$\cdot$0 \\\\
N-LASSO   && $-5\cdot$3 $\pm$ 2$\cdot$0 &&36$\cdot$6 &1$\cdot$0 & $-4\cdot$6 $\pm$ 2$\cdot$0 &&29$\cdot$6 &1$\cdot$0\\
RCV-SIS   &&  0$\cdot$2 $\pm$ 1$\cdot$4 &&50$\cdot$0 &0$\cdot$9 & $-0\cdot$1 $\pm$ 1$\cdot$4 &&50$\cdot$0 &1$\cdot$0\\
RCV-ISIS  &&  0$\cdot$5 $\pm$ 1$\cdot$7 &&30$\cdot$9 &0$\cdot$7 &  0$\cdot$2 $\pm$ 1$\cdot$2 &&29$\cdot$0 &0$\cdot$8\\
RCV-LASSO &&  0         $\pm$ 1$\cdot$3 &&31$\cdot$1 &0$\cdot$9 & $-0\cdot$3 $\pm$ 1$\cdot$1 &&26$\cdot$5 &1$\cdot$0\\
P-SCAD    && $-1\cdot$4 $\pm$ 1$\cdot$1 &&30$\cdot$0 &1$\cdot$0 & $-1\cdot$2 $\pm$ 1$\cdot$7 &&29$\cdot$9 &1$\cdot$0\\
CV-SCAD   &&  0$\cdot$7 $\pm$ 1$\cdot$2 &&30$\cdot$0 &1$\cdot$0 &  0$\cdot$9 $\pm$ 1$\cdot$3 &&29$\cdot$9 &1$\cdot$0\\
P-LASSO   && $-0\cdot$8 $\pm$ 2$\cdot$1 &&36$\cdot$5 &1$\cdot$0 & $-0\cdot$9 $\pm$ 1$\cdot$5 &&29$\cdot$6 &1$\cdot$0\\
CV-LASSO  &&  1$\cdot$4 $\pm$ 1$\cdot$1 &&36$\cdot$5 &1$\cdot$0 &  0$\cdot$8 $\pm$ 1$\cdot$0 &&29$\cdot$6 &1$\cdot$0\\
\end{tabular}}
\label{simu2}
\begin{tabnote}
PMLE, $\ell_1$ penalized maximum likelihood estimator; BC, bias-corrected estimator; MCP, minimax concave penalty; SCAD, smoothly clipped absolute deviation penalty; N, naive; RCV, refitted cross-validation; SIS, sure independent screening; ISIS, iterative SIS; P, plug-in method with degrees-of-freedom correction; CV, cross-validation; 
AMS, average model size; SSP, relative frequency of sure screening
\end{tabnote}
\end{table}

\section{Numerical results}
\subsection{Simulation study}
In this section, we present some simulation results to compare five methods:
the scaled penalized methods with the $\ell_1$ penalty, the minimax concave penalty and
the smoothly clipped absolute deviation penalty,
the $\ell_1$ penalized maximum likelihood estimator \citep{StadlerBG10}, and
its bias correction \citep{SunZ10}.
The least squares estimator after model selection by these five methods is also studied.
The penalized maximum likelihood estimator is
\bes
\big(\hbbeta^{(\mathrm{pmle})},\hsigma^{(\mathrm{pmle})}\big) = \argmax_{\bbeta,\sigma}\Big(
 -\frac{{|}\by-\bX\bbeta{|}_2^2}{2\sigma^2 n}
 -\log\sigma-\lam_0\frac{{|}\bbeta{|}_1}{\sigma}\Big),
\ees
or equivalently the limit of the iteration
$\hsigma \leftarrow \{\by'(\by-\bX\hbbeta)/n\}^{1/2}$ and
$\hbbeta \leftarrow \hbbeta(\hsigma\lam_0)$.
The bias-corrected estimator is one iteration of (\ref{gen-alg}) with (\ref{alg}) from
$(\hbbeta^{(\mathrm{pmle})},\hsigma^{(\mathrm{pmle})})$ with $a=0$,
\bes
\hsigma^{(\mathrm{bc})} = {|}\by - \hbbeta(\hsigma^{(\mathrm{pmle})}\lam_0){|}_2/n^{1/2},\
\hbbeta^{(\mathrm{bc})} = \hbbeta(\hsigma^{(\mathrm{bc})}\lam_0).
\ees
Two simulation examples 
are considered.

\begin{example}
We compare the five estimators at three penalty levels $\lam_j=\surd\{2^{j-1}(\log p)/n\}$, $j=1,2, 3$. The experiment has the setting of Example 2 in \citet{FanGH10}, with the smallest signal, $b=1/\surd3$. We provide their description of the simulation setting in our notation as follows: $\bX$ has independent and identically distributed Gaussian rows with marginal distribution $N(0,1)$, $\mathrm{corr}(x_i,x_j)={r_0}$
for $1\le i<j\le 50$ and $\mathrm{corr}(x_i,x_j) = 0$ otherwise,
$(n,p)=(200,2000)$, nonzero coefficients $\beta_j=1/\surd{3}$ for $j\in S=\{1,2,3\}$, and $\by-\bX\bbeta\sim N(0,\sigma^2\bI)$ with $\sigma=1$. Two configurations are considered: independent columns $\bx_j$ with ${r_0}=0$ and correlated first 50 columns $\bx_j$ with ${r_0}=0.5$. We set  $\gamma=2/(1-\max|\bx_k'\bx_j|/n)$ for the concave penalties.

The top section of Table \ref{simu2} presents our simulation results, while the bottom section
includes the simulation results of \citet{FanGH10} for several joint estimators of $(\bbeta,\sigma)$
using cross-validation, without repeating their experiment.
In addition to the bias and the standard error of the ratios $\hsigma/\sigma$ for the five original estimators
and $\hhsigma/\sigma$ for the least squares estimation after model selection,
we report the average model size $|\Shat|$ and the relative frequency of sure screening, $\Shat\supseteq S$,
as in \citet{FanGH10}, where $\Shat = \supp(\hbeta)$ is the selected model.

Without post processing, the scaled minimax concave penalized selector with the universal penalty level
$\lam_2=\surd\{(2/n)\log p\}$ clearly outperforms other procedures in this example.
However, the results of the least squares estimation after model selection at penalty level $\lam_2$ are nearly
identical to the top performer for all five methods.
In view of the results in average model size and sure screening proportion, the success of post processing at $\lam_2$
is clearly due to the success of model selection.
The five methods select too few variables at the larger penalty level $\lam_3$ and too many at the smaller $\lam_1$,
both leading to substantial bias in the estimation of $\sigma$ for $r_0=0$. For $r_0=0.5$, selecting a
slightly smaller model does not harm so much since a substantial portion of the effect of the missing variables is
explained by the selected variables correlated to them.
The minimax concave penalized selector is nearly unbiased in this example, so that it does not need post processing.
Cross-validation methods select about 30 variables when the true model size is 3.
This over selection is probably the reason for the large bias for most cross-validation methods
and large standard error for all of them.
\end{example}

\begin{table}
\def~{\hphantom{0}}
\captionsetup{width=4.5in}
\tbl{Performance of five methods in Example 2 at penalty levels
$
\lam_j=\{2^{j-1}(\log p) /n\}^{1/2}$ $(j=1,2,3)$,
across 100 replications, in terms of the bias ($\times10$) and standard error ($\times10$) of
$\hsigma/\sigma$ and $\hhsigma/\sigma$,
the false positive, and the false negative}{
\begin{tabular}{c c cccc cccc}\\
&&\multicolumn{4}{c}{${r_0}=0.1$}&\multicolumn{4}{c}{${r_0}=0.9$}\\
Method & &\multicolumn{1}{c}{$\hsigma/\sigma$} & $\hhsigma/\sigma$ & FP & FN
         &\multicolumn{1}{c}{$\hsigma/\sigma$} & $\hhsigma/\sigma$ & FP & FN  \\\\
        &$\lam_1$ & 5.5$\pm$0.3 & 0.2$\pm$0.3 & 3 & 11 & 2.4$\pm$0.3 & $-$0.4$\pm$0.3 & 4 & 12\\
PMLE    &$\lam_2$ & 7.7$\pm$0.4 & 2.1$\pm$0.6 & 0 & 19 & 3.7$\pm$0.3 & $-$0.3$\pm$0.3 & 1 & 15\\
        &$\lam_3$ & 9.5$\pm$0.4 & 6.3$\pm$1.1 & 0 & 30 & 5.7$\pm$0.3 & $-$0.1$\pm$0.3 & 0 & 20 \\\\
        &$\lam_1$& 3.2$\pm$0.3 & $-$0.3$\pm$0.4 & 8 & 9 & 0.3$\pm$0.3 & $-$0.7$\pm$0.3 & 9 & 11 \\
BC      &$\lam_2$ & 6.1$\pm$0.5 & 1.6$\pm$0.5 & 0 & 17 & 1.2$\pm$0.3 & $-$0.3$\pm$0.3 & 2 & 13 \\
        &$\lam_3$ & 9.1$\pm$0.5 & 6.1$\pm$1.1 & 0 & 29 & 3.1$\pm$0.4 & $-$0.1$\pm$0.3 & 0 & 18 \\\\
Scaled  &$\lam_1$ & 1.9$\pm$0.4 & $-$0.8$\pm$0.4 & 14 & 7 & 0.0$\pm$0.3 & $-$0.8$\pm$0.3 & 11 & 11\\
lasso   &$\lam_2$ & 5.0$\pm$0.5 & 1.3$\pm$0.5 & 0 & 16 & 0.6$\pm$0.3 & $-$0.3$\pm$0.3 & 2 & 13 \\
        &$\lam_3$ & 8.9$\pm$0.6 & 6.0$\pm$1.2 & 0 & 29 & 1.8$\pm$0.4 & $-$0.2$\pm$0.3 & 0 & 16 \\\\
Scaled  &$\lam_1$ & $-$0.2$\pm$0.4 & $-$1.1$\pm$0.4 & 14 & 7 & 0.0$\pm$0.3 & $-$0.7$\pm$0.3 & 10 & 19\\
MCP     &$\lam_2$ & 1.8$\pm$0.5 & 0.7$\pm$0.5 & 0 & 13 & 0.6$\pm$0.3 & $-$0.2$\pm$0.3 & 1 & 20 \\
        &$\lam_3$ & 7.8$\pm$1.0 & 5.6$\pm$1.3 & 0 & 28 & 1.7$\pm$0.4 & 0.0$\pm$0.3 & 0 & 22 \\\\
Scaled  &$\lam_1$ & 0.6$\pm$0.4 & $-$1.0$\pm$0.4 & 14 & 7 & 0.0$\pm$0.3 & $-$0.8$\pm$0.3 & 10 & 15 \\
SCAD    &$\lam_2$ & 4.7$\pm$0.6 & 1.3$\pm$0.5 & 0 & 16 & 0.6$\pm$0.3 & $-$0.3$\pm$0.3 & 2 & 15 \\
        &$\lam_3$ & 8.9$\pm$0.6 & 6.0$\pm$1.2 & 0 & 29 & 1.8$\pm$0.4 & $-$0.2$\pm$0.3 & 0 & 17
\end{tabular}}
\label{simu1}
\begin{tabnote}
PMLE, $\ell_1$ penalized maximum likelihood estimator; BC, bias-corrected estimator; MCP, minimax concave penalty; SCAD, smoothly clipped absolute deviation penalty; FP, false positive; FN, false negative
\end{tabnote}
\end{table}

\begin{example}
This experiment has the same setting as in the simulation study in \citet{SunZ10}, where the scaled lasso and
the scaled minimax concave penalized selection are called the naive estimators.
We provide the description of the simulation settings in \citet{SunZ10} in our notation as follows:
$(n,p)=(600,3000)$, 
the $\bx_j$ are normalized columns from a Gaussian random matrix with independent and identically distributed rows and correlation ${r_0}^{|k-j|}$ between the $j$-th and $k$-th entries within each row, $\gamma=2/(1-\max|\bx_k'\bx_j|/n)$ for the  minimax concave penalty and smoothly clipped absolute deviation penalty, the nonzero $\beta^*_j$ are composed of five blocks of $\beta_*(1,2,3,4,3,2,1)'$ centered at random multiples $j_1,\dots,j_5$ of 25, $\beta_*$ sets ${|}\bX\bbeta^*{|}_2^2=3n$, and $\by-\bX\bbeta^*$ is a vector of independent and identically distributed $N(0,1)$ variables.
Thus, the true noise level is $\sigma=1$.
We set ${r_0}=0.1$ for low correlation between design vectors and
${r_0}=0.9$ for high correlation.

\begin{table}
\def~{\hphantom{0}}
\tbl{Estimated coefficients ($\times10^3$) of selected probe sets by four methods in the real data example: the lasso with cross-validation, the lasso with adjusted cross-validation, and the scaled lasso and minimax concave penalized selection at $\lam_0=\lam_2=\{2(\log p)/n\}^{1/2}$}{
\begin{tabular}{lrrrrrrrr}\\
Probe ID 	          & \multicolumn{2}{c}{C-V lasso} & \multicolumn{2}{c}{C-V lasso/LSE} & \multicolumn{2}{c}{Scaled lasso}  & \multicolumn{2}{c}{Scaled MCP}\\
\#cov                          & 200         &  3000         & 200           & 3000         & 200            & 3000          & 200           & 3000 \\
1369353\_at 	               &  $-9\cdot12$ & $-7\cdot$13* & $-7\cdot$09  & $-2\cdot$79* &$-7\cdot$3 	 & $-4\cdot$03*	 &   	 &   \\
1370052\_at$^{\vartriangle}$   &   	          &  3$\cdot$65  &              &              &   	         &   	         &   	 &   \\
1370429\_at 	               &  $-3\cdot$22 &              & $-8\cdot$94* & $-11\cdot$06 & $-8\cdot$78*& $-9\cdot$36 & $-16\cdot$37*&\\
1371242\_at 	               &  $-6\cdot$66 &   	         &              &              &   	         & &   	 &   \\
1374106\_at 	               &  8$\cdot$88* & 10$\cdot$58* & 7$\cdot$33*  & 6$\cdot$14*   & 7$\cdot$45*& 7$\cdot$01*& 8$\cdot$47*& 10$\cdot$02* \\
1374131\_at 	               &  4$\cdot$07  & 0$\cdot$80   &               &              &  	         &   	         &   	 &   \\
1375585\_at$^{\vartriangle}$   &   	          &   	         &               &              &   	     &  0$\cdot$58 	 &   	 &   \\
1384204\_at 	               &   	          &   	         & 0$\cdot$70    &              & 0$\cdot$70 &       	     &   	 &   \\
1387060\_at$^{\vartriangle}$   &   	          &  3$\cdot$50* &               &              &   	     &   	         &   	 &   \\
1388538\_at$^{\vartriangle}$   &   	          &  1$\cdot$42  &               &              &   	     &               &   	 &   \\
1389584\_at 	               &  17$\cdot$16* & 25$\cdot$39* & 20$\cdot$07* & 19$\cdot$61* & 19$\cdot$97* &  21$\cdot$18* &  45$\cdot$75* 	 &  50$\cdot$49* \\
1393979\_at 	               & $-1\cdot$81 &   	         & $-0\cdot$22   &              & $-0\cdot$4 &   	         &   	 &   \\
1379079\_at$^{\vartriangle}$   &   	         & $-1\cdot$43*  &               &              &   	     &   	         &   	 &   \\
1379495\_at 	               &   	         & 4$\cdot$84 	 & 1$\cdot$73    &              & 1$\cdot$71 &  $1\cdot00$ 	 &   	 &   \\
1379971\_at 	               &  13$\cdot$56* & 13$\cdot$1 & 11$\cdot$19*   &  8$\cdot$81  & 11$\cdot$25*&  9$\cdot$52  &   	 &   \\
1380033\_at 	               &  8$\cdot$69  &   	         & 2$\cdot$76    &              & 2$\cdot$97 &   	         & 6$\cdot$75*&\\
1380070\_at$^{\vartriangle}$   &   	         & $0\cdot$19 	 &               &              &   	     &   	         &   	 &   \\
1381787\_at 	               &  $-2\cdot$05 &   	         & $-2\cdot$01   &              &  $-2\cdot$11 	 &   	     &   	 &   \\
1382452\_at$^{\vartriangle}$   &   	         & 12$\cdot$93   &               &              &   	     & 1$\cdot$63 	 &&12$\cdot$91*\\
1382835\_at 	               &  12$\cdot$64 & 5$\cdot$79   & 3$\cdot$73    &              & 4$\cdot$15 	 &   	         &       &   \\
1383110\_at 	               & 9$\cdot$03* & 19$\cdot$99 & 15$\cdot$10* & 16$\cdot$43 & 14$\cdot$97*	 & 16$\cdot$69  & 15$\cdot$80*	 & 23$\cdot$01* \\
1383522\_at 	               & 3$\cdot$03* &  	         &         *      &              &   *       &   	         &   	 &   \\
1383673\_at 	               &  5$\cdot$54 & 6$\cdot$12* 	 & 6$\cdot$07    & 6$\cdot$15*   & 6$\cdot$08 	 &  6$\cdot$47*	 &	 &   \\
1383749\_at 	               &  $-13\cdot$86 & $-10\cdot$85* & $-10\cdot$84  & $-6\cdot$7*   & $-11\cdot$02 	 &  $-8\cdot$07* & $-2\cdot$74 *	 &  $-1\cdot$11* \\
1383996\_at 	               &  25$\cdot$01* &  17$\cdot$82* & 18$\cdot$61* & 14$\cdot$30*   & 18$\cdot$88* 	 &  15$\cdot$52* & 25$\cdot$07* 	 &  19$\cdot$19* \\
1385687\_at$^{\vartriangle}$   &   	         &  $-0\cdot$99 	 &               &              &   	         &   	         & 	 &   \\
1386683\_at 	               &   	         &   	         &              & 4$\cdot$60*       &               &  2$\cdot$90*   & 	 &   \\
1390788\_a\_at 	               &  0$\cdot$92 &   	         &               &              &   	         &   	         &   	 &   \\
1392692\_at$^{\vartriangle}$   &   	         & 1$\cdot$74 	 &               &              &   	         &   	         &   	 &   \\
1393382\_at 	               &  2$\cdot$43 &   	         &               &              &   	         &   	         &   	 &   \\
1393684\_at 	               &  1$\cdot$59 &   	         &               &              &   	         &  	         &   	 &   \\
1395076\_at$^{\vartriangle}$   &   	         &   	         &               &              &   	         & 0$\cdot$23 	 &   	 &   \\
1397489\_at$^{\vartriangle}$   &   	         & 3$\cdot$33 	 &               &              &   	         &  	         &   	 &   \\\\
Model size                     & 19          & 20            &15&10  &15             & 14            &7             &  6\\
$\widehat{\lam}=\hsigma\lam_0$ & 0.0103      & 0.0163        &0.025& 0.035  &0.0243         & 0.0315        &0.0244        & 0.0304\\
\end{tabular}
}
\label{varselect}
\begin{tabnote}
C-V lasso/LSE, the lasso estimator with the adjusted cross-validation; MCP, minimax concave penalty;
\#cov, the number of covariates considered; $\vartriangle$, probes not in the smaller set of 200
probes; * , covariates selected by stability selection
\end{tabnote}
\end{table}

We summarize the simulation results in Table \ref{simu1}, which provides the bias and standard error of the ratios
$\hsigma/\sigma$ for the selector and $\hhsigma/\sigma$ for the least squares estimator after model selection,
the false positive $|\Shat\setminus S|$, and the false negative $|S\setminus \Shat|$.
Without post processing,
the scaled lasso outperforms the penalized maximum likelihood estimator and its bias correction,
which are also based on the lasso path. However, the scaled lasso estimate of $\sigma$ is still biased,
and the level of bias is comparable with the order of the error bound $(|S|/n)\log p= 0.47$ in (\ref{th-2-3}).
This and the failure in sure screening by any method reflect the difficulty of this example,
where $|S|=35$ is not small and the signal is weak, with average $\beta_*=0.11$
and $0.05$ respectively for $r_0=0.1$ and $r_0=0.9$.
From this perspective, the scaled minimax concave penalized selection, designed to reduce the bias
of the lasso, performs quite well at the universal penalty level $\lam_2=\surd\{(2/n)\log p\}$, especially with post processing.
The least squares estimation after model selection reduces the bias substantially in all cases,
even without successful model selection.
This example seems to suggest the possibility of improving the performance of the scaled estimators
at a penalty level $\lam$ smaller than the universal penalty level $\lam_2$, a simple
upper bound for $|X'(y-X\beta)/(\sigma^* n)|_\infty$ under $\mathrm{pr}_{\bbeta,\sigma}$. However, consistent
variable selection requires $\lam\ge |X'(y-X\beta^*)/(\sigma^* n)|_\infty$ as Example 1 demonstrates.
Since the scaled lasso estimator $\hsigma$ is an increasing function of the penalty level by Proposition \ref{prop-1},
it is always possible to reduce the bias of $\hsigma$ to zero by taking a specific $\lam$ for each
specific example. However, the two examples in our simulation experiment demonstrate the difficulty of
picking such a penalty level consistently.
\end{example}

\subsection{Real data example}
We study a data set containing 18976 probes for 120 rats, which is reported in \citet{Scheetz06}. Our goal is to find probes that are related to that of gene TRIM32, which has been found to cause Bardet--Biedl syndrome, a genetically heterogeneous disease of multiple organ systems including the retina. We consider linear regression with the probe from TRIM32, 1389163\_at, as the response variable.
As in \citet{HuangMZ08}, we focus on 3000 probes with the largest variances among the 18975 covariates and consider two approaches.
The first approach is to regress on these $p=3000$ probes. The second approach is to regress on the 200 probes among the 3000 with the largest marginal correlation coefficients with TRIM32. For the cross-validation lasso, we randomly partition the data 1000 times, each with a training set of size 80 and a validation set of size 40. For each partition, the penalty level $\lam$ is selected by minimizing the prediction mean squared error in the validation set. Then we compute the lasso estimator with all 120 observations
at the penalty level equal to the median of the selected penalty levels with the 1000 random partitions. Since cross-validation tends to choose a larger model, we also consider an adjusted version using the cross-validated error of the least squares estimator after the lasso selection. For the minimax concave penalty, we set $\gamma=2/(1-\sigma_{0.95})=6.37$, where $\sigma_{0.95}$ is the 95\% quantile of $|\bx'_k\bx_j|/n$.

\begin{figure}
\begin{center}
\includegraphics[width=4.5in,height=2.3in]{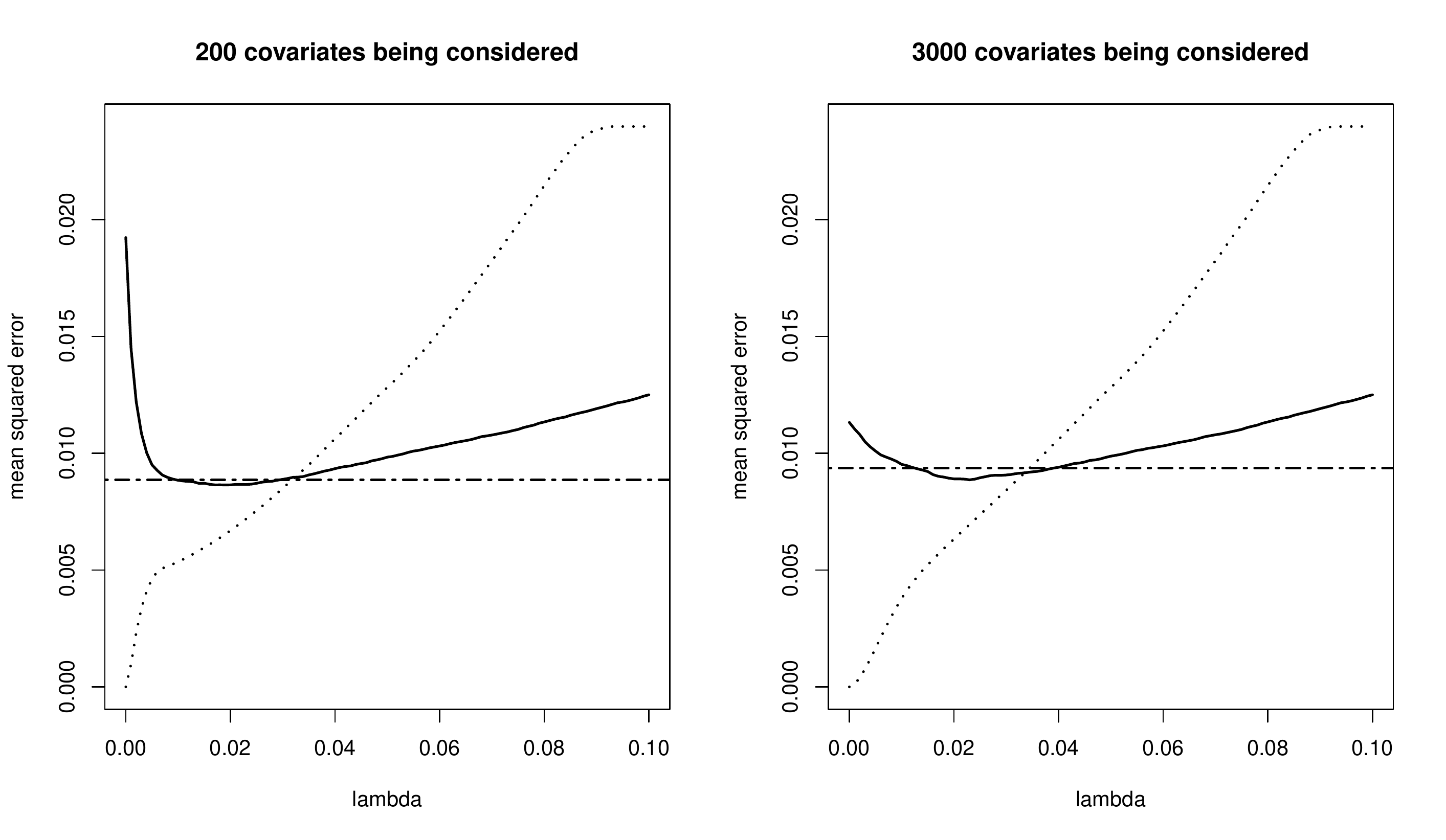}
\captionsetup{width=4.5in}
\caption{Mean squared prediction error against penalty level:
solid curve, testing error of the lasso for fixed $\lam$;
dotted curve, training error of the lasso for fixed $\lam$;
dot-dashed line, testing error of the scaled lasso with fixed $\lam_0=\surd\{(2/n)\log p\}$,
or equivalently the lasso at penalty level $\hsigma\lam_0$.}
\label{mse-lam}
\end{center}
\vspace{-.2in}
\end{figure}

\begin{table}
\def~{\hphantom{0}}
\captionsetup{width=4.5in}
\tbl{Prediction performance of eight methods in the real data example at penalty levels
$\lam_0=\lam_j=\{2^{j-1}(\log p) /n\}^{1/2} (j=1,2,3)$, in terms of the prediction mean squared
error ($\times10^2$), estimated model size, and correlation coefficient ($\times10^2$)
between fitted and observed responses}{
\begin{tabular}{ccrrrrrr}\\
&& \multicolumn{3}{c}{\#cov = 200} & \multicolumn{3}{c}{\#cov = 3000} \\
Method&& P-MSE & ${|}\hbbeta{|}_0$ & corr & P-MSE & ${|}\hbbeta{|}_0$&corr\\\\
PMLE&$\lam_1$	& 0$\cdot$94 	& 12 	& 67$\cdot$1 	& 0$\cdot$97 	& 12 	& 63$\cdot$5 	\\
&$\lam_2$	& 0$\cdot$97 	& 9 	& 63$\cdot$5 	& 1$\cdot$04 	& 7 	& 59$\cdot$8 	\\
&$\lam_3$	& 1$\cdot$09 	& 6 	& 57$\cdot$6 	& 1$\cdot$23 	& 3 	& 52$\cdot$2 	\\
\\
BC&$\lam_1$	& 0$\cdot$93 	& 13 	& 68$\cdot$2 	& 0$\cdot$96 	& 15 	& 64$\cdot$6 	\\
&$\lam_2$	& 0$\cdot$95 	& 10 	& 64$\cdot$7 	& 1$\cdot$01 	& 9 	& 60$\cdot$9 	\\
&$\lam_3$	& 1$\cdot$04 	& 7 	& 59$\cdot$4 	& 1$\cdot$17 	& 4 	& 53$\cdot$1 	\\
\\
Scaled lasso&$\lam_1$	& 0$\cdot$93 	& 13 	& 68$\cdot$4 	& 0$\cdot$96 	& 17 	& 64$\cdot$3 	\\
&$\lam_2$	& 0$\cdot$94 	& 10 	& 65$\cdot$2 	& 0$\cdot$98 	& 10 	& 61$\cdot$7 	\\
&$\lam_3$	& 1$\cdot$02 	& 7 	& 60$\cdot$8 	& 1$\cdot$13 	& 5 	& 53$\cdot$9 	\\
\\
Scaled MCP&$\lam_1$	& 1$\cdot$03 	& 6 	& 66$\cdot$4 	& 1$\cdot$08 	& 8 	& 62$\cdot$3 	\\
&$\lam_2$	& 1$\cdot$03 	& 5 	& 63$\cdot$4 	& 1$\cdot$06 	& 5 	& 60$\cdot$0 	\\
&$\lam_3$	& 1$\cdot$12 	& 3 	& 59$\cdot$1 	& 1$\cdot$18 	& 2 	& 54$\cdot$9 	\\
\\
Scaled SCAD&$\lam_1$	& 1$\cdot$00 	& 11 	& 68$\cdot$9 	& 1$\cdot$01 	& 14 	& 65$\cdot$1 	\\
&$\lam_2$	& 0$\cdot$95 	& 10 	& 68$\cdot$8 	& 0$\cdot$98 	& 10 	& 65$\cdot$9 	\\
&$\lam_3$	& 1$\cdot$01 	& 8 	& 65$\cdot$0 	& 1$\cdot$09 	& 5 	& 59$\cdot$7 	\\
\\
C-V lasso&	& 0$\cdot$94 	& 15 	& 69$\cdot$0 	& 0$\cdot$99 	& 25 	& 63$\cdot$8 	\\
C-V lasso/LSE1 &    & 0$\cdot$97 	& 11 	& 64$\cdot$8 	& 0$\cdot$98 	& 12 	& 62$\cdot$5 	\\
C-V lasso/LSE2 &    & 0$\cdot$97 	& 11 	& 66$\cdot$8 	& 1$\cdot$09 	& 12 	& 62$\cdot$6 	\\
\end{tabular}
}
\label{tab-pred}
\begin{tabnote}
PMLE, $\ell_1$ penalized maximum likelihood estimator; BC, bias-corrected PMLE;
MCP, minimax concave penalty; SCAD, smoothly clipped absolute deviation penalty;
C-V lasso/LSE1, the lasso with adjusted cross-validation;
C-V lasso/LSE2, the least squares estimator after the lasso selection with adjusted cross-validation,
\#cov, the number of covariates considered; corr, the correlation coefficient between
fitted and observed responses;  P-MSE, prediction mean squared error
\end{tabnote}
\end{table}

Table \ref{varselect} shows the probe sets identified by four methods: the cross-validation lasso, its adjusted version,
the scaled lasso at at universal penalty level $\lam_2=\{2(\log p)/n\}^{1/2}$,
and the minimax concave penalized selection at the same penalty level.
We apply stability selection \citep{MeinshausenB10} to check the reliability of selection.
Let $W_1,\dots,W_p$ be independent variables with $P(W=0.2)=P(W=1)=1/2$ and
\bes
\hbbeta^{W}=\arg\min_{b}\frac{|\by-\bX b|_2^2}{2n} + \hlam\sum_{j=1}^p |\beta_j|/W_j,
\ees
where $\hlam$ is the penalty level chosen by individual methods.
Stability selection selects variables with nonzero estimated $\hbeta^{W}_j$ over 50 times in 100 replications.
We observe that the scaled minimax concave penalized selector produces most sparse and most
stable selection, followed by the adjusted cross-validation, the scaled lasso and then the plain cross-validation.
The selection results are consistent among the four methods in the sense that the selected models are almost nested.
Since the model size is between 6 and 8 by stability selection in all 8 cases and by the scaled minimax
concave penalized selection for both $p=200$ and $p=3000$, these two methods provide most consistent results.
The scaled lasso and the adjusted cross-validation yield identical
lasso and stability selections for $p=200$ and identical stability selection for $p=3000$.

We also compare the prediction performance of the scaled lasso with that of the lasso with the best fixed penalty level. We compute the scaled estimators in 1000 replications. In each replication, the dataset is split at random into a training set with 80 observations and a test set with 40 observations. The prediction mean squared error is computed within the test set, while the scaled estimators and the lasso estimator with fixed penalty level $\lam$ are computed based on the training set. Figure \ref{mse-lam} demonstrates that in prediction, the scaled lasso with $\lam_0$ chosen as $\lam_2=\{2(\log p)/n\}^{1/2}$ performs almost as well as the lasso with the optimal fixed $\lam$.

In addition, we compare the prediction performance of all the estimators mentioned in this section. In each replication, we compute the penalized maximum likelihood estimator, its bias-correction, and scaled penalization methods based on the training set of 80 observations. For cross-validation, the training set of 80 observations is further partitioned at random 100 times into two groups of sizes 60 and 20, and a penalty level is selected by minimizing the estimated loss in the smaller group for the lasso estimator based on the larger group. This selected penalty level is then used for the lasso with the entire training set. Thus, the cross-validation lasso is also based on the training set with 80 observations. For the penalty level selected by the adjusted cross-validation, two estimators are considered: the lasso estimator
and the least squares estimator after the lasso selection.
In Table \ref{tab-pred}, we present the medians of the prediction mean squared error and the selected model size in the 200 replications.
The scaled lasso has comparable prediction performance as cross-validation. Again, Table \ref{tab-pred} suggests that original cross-validation tends to choose larger models, while adjusted cross-validation leads to  results comparable with the scaled lasso.

\section{Discussion}

In the theoretical analysis, we have considered $\lam_0=A\{(2/n)\log p\}^{1/2}$ with $A>1$.
This choice is somewhat conservative from a number of points of view.
Simulation results suggest that the requirement $A>1$ is a mathematical technicality.
If $|X'\varepsilon/n|_\infty\le\lam_*$ with large probability for a standard normal
vector $\varepsilon$, the theoretical results in this paper are all valid under
$\mathrm{pr}_{\bbeta,\sigma}$ when $\lam_0$ is replaced by the smaller $\min(\lam_0,A\lam_*)$.
The value of $\lam_*$ can be estimated by simulation with the given $X$ and separately generated $\varepsilon$.
 A somewhat sharper theoretical choice of $\lam_0$ is $A\{(2/n)\log(p/s)\}^{1/2}$ with
the unknown $s=|\beta^*|_0$ \citep{Zhang10-mc+}, or its simulated version with
$\lam_*=\max_{|T|=s}|X_T'\varepsilon|_2/|T|^{1/2}$.
The difference between the two $\lam_0$ is limited unless $\log p = \{1+o(1)\}\log n$.
A reviewer called our attention to an unpublished 2011 report by Baraud, Giraud and Huet,
available at http://arxiv.org/abs/1007.2096,
whose method can be used to select a penalty level to nearly minimize the order of a penalized prediction error.
This may also justify the use of smaller estimated penalty levels.

In the proof of our theoretical results for the scaled lasso, we use oracle inequalities for fixed penalty which unify and somewhat sharpen existing results. We now present this result. Define
\bel{eta^*}
\eta^*(\lam,\xi) = \min_T\ 2^{-1} \Big[\eta(\lam,\xi,\bbeta^*,T)
+\big\{\eta^2(\lam,\xi,\bbeta^*,T)
-16\lam^2{|}\bbeta^*_{T^c}{|}_1^2\big\}^{1/2}\Big]
\eel
as a sharper version of $\eta(\lam,\xi,\bbeta^*,T)$ in (\ref{eta}).

{\theorem\label{th-3}
Let $\hbbeta(\lam)$ be the minimizer of (\ref{pen-loss-beta}) with $\rho(t)=t$.
Let $\bbeta^*\in \real^p$ be a target vector and $\xi>1$.  Then, in the event ${|}\bX'(\by-\bX\bbeta^*){|}_\infty/n \le \lam(\xi-1)/(\xi+1)$, we have
\bel{th-3-1}
{|}\bX\hbbeta(\lam)-\bX\bbeta^*{|}_2^2/n\le
\min\big\{\eta_*(\lam,\xi),\eta^*(\lam,\xi)\big\}
\eel
with $\eta_*(\lam,\xi)$ in (\ref{min-eta}). Moreover, in the same event and with  $\mu(\lam,\xi)$ in (\ref{mu}),
\bel{th-3-2}
{|}\hbbeta(\lam)-\bbeta^*{|}_1\le \mu(\lam,\xi).
\eel
}

The interpretations of (\ref{th-3-1}) and (\ref{th-3-2}) are given in (\ref{th-1-3}) and (\ref{ell_1-rate}), along with their relationship to several existing results.  We note here that the
condition $\kappa(\xi,S)\asymp 1$ for (\ref{th-1-3}) and (\ref{ell_1-rate}), weaker than
the parallel condition on the restricted eigenvalue \citep{BickelRT09}, can be slightly weakened
by using $F_1(\xi,S)$ in (\ref{scif}) \citep{YeZ10}.

\section*{Acknowledgement}
This research was supported by the National Science Foundation and the National Security Agency. We thank Jian Huang for sharing the gene expression data, and reviewers for valuable suggestions.

\appendix
\section*{Appendix}

Here we prove Proposition \ref{prop-1}, Theorem \ref{th-3}, Theorem \ref{th-1}, Theorem \ref{th-2} and then
Theorem \ref{th-mleas}.

\medskip
\begin{proof}[of Proposition \ref{prop-1}]
(i) Since $\hbbeta=\hbbeta(\sigma\lam_0)$ is a solution of (\ref{KKT})
at $\lam=\sigma\lam_0$,
\bes
\Big\{(\pa/\pa\bw)L_{\lam_0}(\bw,\sigma)\Big|_{\bw=\hbbeta(\sigma\lam_0)}\Big\}_j = 0,
\ees for all $\hbeta_j(\sigma\lam_0)\neq 0$.
Since $\{j: \hbeta_j(\lam)\}$ is unchanged in a neighborhood of $\sigma\lam_0$, $[(\pa/\pa\sigma)\{\hbbeta(\sigma\lam_0)/\sigma\}]_j=0$
for $\hbeta_j(\sigma\lam_0)=0$. Thus,
\bes
&& \frac{\pa}{\pa\sigma}L_{\lam_0}\{\hbbeta(\sigma\lam_0),\sigma\}
= \frac{\pa}{\pa t}L_{\lam_0}\{\hbbeta(\sigma\lam_0),t\}\Big|_{t=\sigma}
= \frac{1-a}{2} - \frac{{|}\by-\bX\hbbeta(\sigma\lam_0){|}_2^2}{2n\sigma^2}.
\ees
(ii) The convergence of (\ref{gen-alg}) and (\ref{alg}) follows from the joint convexity of $L_{\lam_0}(\beta,\sigma)$.
The scale invariance follows from $L_0(c\bbeta,c\sigma;\bX,c\by)=c L_0(\bbeta,\sigma;\bX,\by)$, where $L_0(\bbeta,\sigma;\bX,\by)$ expresses the dependence of (\ref{pen-loss-joint}) on the data $(\bX,\by)$.
\end{proof}

\medskip
\begin{proof}[of Theorem \ref{th-3}] (i) Let $\hbeta=\hbeta(\lam)$.
Since $\sigma^*z^*={|}\bX'(\by-\bX\bbeta^*){|}_\infty/n$
and ${\dot\rho}(|\hbeta_j|/\lam)=1$ for $\hbeta_j\neq 0$,  the inner product of $w-\hbeta$
and  the Karush--Kuhn--Tucker condition (\ref{KKT}) yield
\bes
(\bX\hbbeta-\bX\bw)'(\bX\hbeta-X\beta^*)/n
\le \lam({|}\bw{|}_1-{|}\hbbeta{|}_1)+ \sigma^* z^*{|}\bw-\hbbeta{|}_1.
\ees
Since $2(\bX\hbbeta-\bX\bw)'(\bX\hbeta-X\beta^*)=
{|}\bX\hbbeta-\bX\bw{|}_2^2+{|}\bX\bh{|}_2^2-{|}\bX\bbeta^*-\bX\bw{|}_2^2$,
this gives the basic inequality (\ref{basic}).
Let $\bh=\hbbeta-\bbeta^*$.
Since $\sigma^*z^*\le\lam(\xi-1)/(\xi+1)$,
$\lam\{{|}\bw{|}_1-{|}\hbbeta{|}_1\}+ \sigma^* z^*{|}\bw-\hbbeta{|}_1$
is no greater than $b{|}(\bw-\hbbeta)_{T}{|}_1
+2\lam{|}\bw_{T^c}{|}_1-(b/\xi){|}(\bw-\hbbeta)_{T^c}{|}_1$ with $b=2\xi\lam/(\xi+1)$.
Thus, (\ref{basic}) implies
\bel{thm-pred-pf}
{|}\bX\hbbeta-\bX\bw{|}_2^2/n+{|}\bX\bh{|}_2^2/n+(2b/\xi){|}(\bw-\hbbeta)_{T^c}{|}_1
\le 2c + 2b{|}(\bw-\hbbeta)_{T}{|}_1
\eel
with $c={|}\bX\bbeta^*-\bX\bw{|}_2^2/(2n)+2\lam {|}\bw_{T^c}{|}_1$.
For $T=\emptyset$ and $w=\beta^*$, (\ref{thm-pred-pf}) directly yields
$|Xh|_2^2/n\le c = 2\lam|\beta^*|_1$.
For general $\{w,T\}$, we want to prove
\bes
{|}\bX\bh{|}_2^2/n \le \eta(\lam,\xi,\bw,T) = 2c + b^2/a,\quad a = \kappa^2(\xi,T)/|T|.
\ees
It suffices to consider ${|}\bX\bh{|}_2^2/n \ge 2c$. In this case,
$\hbbeta-\bw\in \scrC(\xi,T)$ by (\ref{thm-pred-pf}), so that by (\ref{compatible})
\bel{thm-pred-pf-1}
a {|}(\bw-\hbbeta)_{T}{|}_1^2 \le {|}\bX\hbbeta-\bX\bw{|}_2^2/n.
\eel
Let $x={|}(\bw-\hbbeta)_{T}{|}_1$ and $y={|}\bX\bh{|}_2^2/n$.
It follows from (\ref{thm-pred-pf}) and (\ref{thm-pred-pf-1}) that
$ax^2+y \le 2c + 2bx$. For such $(x,y)$, $y-2c\le \max_x\{2bx - ax^2\}=b^2/a$.
This gives $y\le 2c+b^2/a = \eta(\lam,\xi,\bw,T)$.

For $\bw=\bbeta^*$, it suffices to consider the case
$y > c=2\lam{|}\bbeta^*_{T^c}{|}_1$, where
the cone condition holds for $\hbeta-\beta^*$.
Now, $(x,y)$ satisfies $ax^2\le y\le c+bx$.
The maximum of $y$, attained at $ax^2=c+bx$, is
\bel{thm-pred-pf-2}
c+b\{b+(b^2+4ac)^{1/2}\}/(2a)
= \big[\eta(\lam,\xi,\bbeta^*,T) +\{\eta^2(\lam,\xi,\bbeta^*,T)-4c^2\}^{1/2}\big]/2.
\eel

(ii) Let $0< \nu  <1$ and $T\subset\{1,\ldots,p\}$.
It follows from (\ref{thm-pred-pf}) with $\bw=\bbeta^*$ that
\bes
(1+\xi){|}\bX\bh{|}_2^2/n + 2\lam {|}\bh_{T^c}{|}_1
\le 2\lam(\xi+1){|}\bbeta^*_{T^c}{|}_1+2\xi\lam{|}\bh_T{|}_1.
\ees
It suffices to consider $\nu  |\bh|_1\ge (\xi+1)|\bbeta^*_{T^c}|_1$. In this case
\bes
(1+\xi){|}\bX\bh{|}_2^2/n + 2\lam(1-\nu ){|}\bh_{T^c}{|}_1
\le 2\lam(\xi+\nu){|}\bh_T{|}_1.
\ees
Thus, $(1-\nu ){|}\bh_{T^c}{|}_1\le (\xi+\nu ){|}\bh_{T}{|}_1$, or equivalently
$\bh\in \scrC\{(\xi+\nu )/(1-\nu ),T\}$. It follows from (\ref{compatible}) that
${|}\bX\bh{|}_2^2/n\ge {|}\bh_T{|}_1^2 \kappa^2\{(\xi+\nu )/(1-\nu ),T\}/|T|$,
so that
\bel{basic1a1}
(1+\xi){|}\bh_T{|}_1^2\kappa^2\{(\xi+\nu )/(1-\nu ),T\}/|T| + 2(1-\nu )\lam{|}\bh_{T^c}{|}_1
\le 2(\xi+\nu )\lam{|}\bh_T{|}_1.
\eel
Let $x={|}\bh_T{|}_1$ and $y={|}\bh_{T^c}{|}_1$. Write (\ref{basic1a1}) as
$ax^2+by \le cx$. Subject to this inequality, the maximum of $x+y$
is $\max_{x\ge 0}\{x+(cx-ax^2)/b\}$. This maximum, attained at
$2ax = b+c$, is $x(b+c)/(2b)=(b+c)^2/(4ab)$. Thus,
\bes
{|}\bh{|}_1\le \frac{\{2(\xi+1)\lam\}^2 |T|}{4(1+\xi)\kappa^2\{(\xi+\nu )/(1-\nu ),T\}\{2(1-\nu)\lam\}}
= \frac{(\xi+1)\lam |T|/(1-\nu)}{2 \kappa^2\{(\xi+\nu )/(1-\nu ),T\}}.
\ees
This gives ${|}\bh{|}_1\le \mu(\lam,\xi)$ for $\nu  |\bh|_1\ge (\xi+1)|\bbeta^*_{T^c}|_1$.
\end{proof}

\begin{proof}[of Theorem \ref{th-1}]
Assume $\tau_0<1$ without loss of generality.
Consider $t\ge\sigma^*(1-\tau_0)$
and the penalty level $\lam=t\lam_0$ for the lasso.
Since $z^*\sigma^*\le \sigma^*(1-\tau_0)\lam_0(\xi-1)/(\xi+1)\le \lam(\xi-1)/(\xi+1)$
and $\sigma^*={|}\by-\bX\bbeta^*{|}_2\big/n^{1/2}$,
the Cauchy--Schwarz inequality and (\ref{th-3-1}) imply
\bes
\big|{|}\by-\bX\hbbeta({t}\lam_0){|}_2\big/n^{1/2} - \sigma^*\big|
\le {|}\bX\hbbeta({t}\lam_0)-\bX\bbeta^*{|}_2\big/n^{1/2}
\le \eta_*^{1/2}(t\lam_0,\xi).
\ees
Since $\eta_*^{1/2}(t\lam_0,\xi)\le \sigma^*\tau_0$ for $t<\sigma^*$,
the derivative (\ref{prop-1-1}) of the loss with $a=0$ satisfies
\bes
2{t}^2\frac{\pa}{\pa{t}}L_{\lam_0}\{\hbbeta({t}\lam_0),{t}\}
= {t}^2 - {|}\by-\bX\hbbeta({t}\lam_0){|}_2^2/n
\le t^2-(\sigma^*)^2(1-\tau_0)^2 = 0\ \hbox{ at $t=\sigma^*(1-\tau_0)$}.
\ees
This implies $\hsigma \ge \sigma^*(1-\tau_0)$
by the strict
convexity of the profile loss (\ref{pen-loss-joint}) in $\sigma$. For $t>\sigma^*$,
$\eta_*^{1/2}(t\lam_0,\xi)\le t\tau_0$ by (\ref{eta}) and (\ref{min-eta}), so that at $t=\sigma^*/(1-\tau_0)$,
\bes
{t}^2 - {|}\by-\bX\hbbeta({t}\lam_0){|}_2^2/n
\ge t^2 - \big(\sigma^*+t\tau_0\big)^2 \ge 0.
\ees
This implies $\sigma^*\ge \hsigma(1-\tau_0)$ by the strict convexity of (\ref{pen-loss-joint}) in $\sigma$.
Thus, the first part of (\ref{th-1-1}) holds. Moreover,
\bes
{|}\bX\hbbeta-\bX\bbeta^*{|}_2\big/n^{1/2}
\le \eta_*^{1/2}(\hsigma\lam_0,\xi)\le \eta_*^{1/2}\{\sigma^*\lam_0/(1-\tau_0),\xi\}
\le\sigma^*\tau_0/(1-\tau_0).
\ees
Finally, since $\mathrm{pr}_{\bbeta,\sigma}[{|}\bX'(\by-\bX\bbeta)/n{|}_\infty\le\sigma\{(2/n)\log p\}^{1/2}]\to 1$,
(\ref{th-1-2}) follows from (\ref{th-1-1}).
\end{proof}

\medskip
The proof of Theorem 2 requires the following lemma.

\begin{lemma}\label{t-dist} Let $T_m$ have the
t-distribution with $m$ degrees of freedom.
Then, there exists $\eps_m\to 0$ such that for all $t>0$
\bel{t-bound}
\mathrm{pr}\big[ T_m^2 > m\{e^{2t^2/(m-1)}-1\} \big]\le (1+\eps_m)e^{-t^2}/(\pi^{1/2}t).
\eel
\end{lemma}

\begin{proof}[of Lemma \ref{t-dist}]
Let $x=[m\{e^{2t^2/(m-1)}-1\}]^{1/2}$. Since $T_m$ has the $t$-distribution,
\bes
\mathrm{pr}\Big( T_m^2> x^2\Big)
&=& \frac{2\Gamma\{(m+1)/2\}}{\Gamma(m/2)(m\pi)^{1/2}}
\int_x^\infty \Big(1+\frac{u^2}{m}\Big)^{-(m+1)/2}{\rm d}u
\cr &\le & \frac{2\Gamma\{(m+1)/2\}}{x\Gamma(m/2)(m\pi)^{1/2}}
\int_x^\infty \Big(1+\frac{u^2}{m}\Big)^{-(m+1)/2}u{\rm d}u
\cr & = & \frac{2\Gamma\{(m+1)/2\}m}{x\Gamma(m/2)(m\pi)^{1/2}(m-1)}
\Big(1+\frac{x^2}{m}\Big)^{-(m-1)/2}.
\ees
Since $x\ge t\{2m/(m-1)\}^{1/2}$,
\bes
\mathrm{pr}\Big( T_m^2> x^2\Big)
\le \frac{\surd{2}\Gamma\{(m+1)/2\}}{\Gamma(m/2)(m-1)^{1/2}}
\frac{e^{-t^2}}{t\pi^{1/2}}
=(1+\eps_m)\frac{e^{-t^2}}{t\pi^{1/2}},
\ees
where $\eps_m = \{2/(m-1)\}^{1/2}\Gamma\{(m+1)/2\}/\Gamma(m/2) - 1\to 0$
as $m\to\infty$.
\end{proof}

\begin{proof}[of Theorem \ref{th-2}]
We need to express $\tau_*^2$ as a function of $\sigma$ at $\sigma=\sigma^*$
in the proof. Define
\bes
\phi(\sigma) = \lam_0\mu(\sigma\lam_0,\xi)/\sigma,\
\phi_+=\frac{\phi(\sigma^*)\xi}{(\xi+1)\{1-\phi(\sigma^*)\}_+},\
\phi_-=\frac{\phi(\sigma^*)(\xi-1)}{\xi+1}.
\ees
We have $\tau_*^2=\phi(\sigma^*)<1$,
$\phi_-\le\phi(\sigma^*)$ and $\phi_+\le \phi(\sigma^*)/(1-\phi(\sigma^*)$.

(i) Consider $z^*\le (1-\phi_-)\lam_0(\xi-1)/(\xi+1)$.
Let $\lam={t}\lam_0$ and $\bh =\hbbeta(\lam) - \bbeta^*$.
Since ${|}\bX'(\by-\bX\bbeta^*)/n{|}_\infty = z^*\sigma^*$,
the Karush--Kuhn--Tucker condition (\ref{KKT}) gives
\bel{pf-th-1-1}
- (z^*\sigma^*+\lam){|}\bh{|}_1 & \le & (\bX\bh)'\{\by-\bX\bbeta^*+ \by- \bX\hbbeta(\lam)\}/n
\cr &=& (\sigma^*)^2 - {|}\by-\bX\hbbeta(\lam){|}_2^2/n
\cr &=& (\bX\bh)'\{2(\by-\bX\bbeta^*) - \bX\bh\}/n \le 2z^*\sigma^*{|}\bh{|}_1
\eel
as lower and upper bounds for $(\sigma^*)^2 - {|}\by-\bX\hbbeta(\lam){|}_2^2/n$. This
is a key point in the proof.

For $t\ge \sigma^*(1-\phi_-)$, $z^*\sigma^*\le{t}\lam_0(\xi-1)/(\xi+1)=\lam(\xi-1)/(\xi+1)$,
so that (\ref{th-3-2}) in Theorem \ref{th-3} implies ${|}\bh{|}_1\le \mu(t\lam_0,\xi)$.
It follows (\ref{pf-th-1-1}) that for $t = \sigma^*(1-\phi_-)$,
\bes
t^2 - {|}\by-\bX\hbbeta(t\lam_0){|}_2^2/n
&\le& t^2 - (\sigma^*)^2 + 2z^*\sigma^*\mu(t\lam_0,\xi)
\cr &\le&  2t(t - \sigma^*)+ 2t \lam_0(\xi-1)(\xi+1)^{-1}\mu(\sigma^*\lam_0,\xi) = 0,
\ees
due to $\phi_-=(\xi-1)(\xi+1)^{-1}\phi(\sigma^*)
= (\xi-1)(\xi+1)^{-1}\lam_0\mu(\sigma^*\lam_0,\xi)/\sigma^*$.
As in the proof of Theorem \ref{th-1}, we find $\hsigma/\sigma^*\ge 1-\phi_-$
by (\ref{prop-1-1}) and the strict convexity of (\ref{pen-loss-joint}) in $\sigma$.

Now we prove that $\hsigma/\sigma^*\le 1+\phi_+$.
For $t>\sigma^*$, $\mu(t\lam_0,\xi)\le (t/\sigma^*)\mu(\sigma^*\lam_0,\xi)$
by (\ref{mu}). Thus, since
$(\xi-1)/(\xi+1) + 1 = 2\phi_+\{1-\phi(\sigma^*)\}/\phi(\sigma^*)$ and $\phi_+\le (1+\phi_+)\phi(\sigma^*)$,
for $t/\sigma^*=1+\phi_+$, (\ref{pf-th-1-1}) and (\ref{th-3-2}) imply that
\bes
t^2 - {|}\by-\bX\hbbeta(t\lam_0){|}_2^2/n
    &\ge& t^2 - (\sigma^*)^2 - (z^*\sigma^*+t\lam_0)\mu(t\lam_0,\xi)
\cr & \ge & (t+\sigma^*)\sigma^*\phi_+  -  \{(\xi-1)/(\xi+1)+1+\phi_+\}t\lam_0\mu(\sigma^*\lam_0,\xi)
\cr & = & (\sigma^*)^2\big((2+\phi_+)\phi_+  -  [2\phi_+\{1-\phi(\sigma^*)\}/\phi(\sigma^*)+\phi_+](1+\phi_+)\phi(\sigma^*)\big)
\cr & = & (\sigma^*)^2\phi_+\{\phi(\sigma^*)(1+\phi_+)-\phi_+\}>0.
\ees
It follows that $\hsigma/\sigma^*\le 1+\phi_+$ by convexity.

Since $1-\phi_-\le \hsigma/\sigma^*\le 1+\phi_+$,
${|}\hbbeta(\hsigma\lam_0)-\bbeta^*{|}_1
\le \mu(\hsigma\lam_0,\xi)\le \mu(\sigma^*\lam_0,\xi)(1+\phi_+)$.
This completes the proof of (\ref{th-2-1}).

(ii) Let $z_j=\bx_j'(\by-\bX\bbeta^*)/(n\sigma^*)$ with $z^*=\max_{j\le p}|z_j|$.
Under $\mathrm{pr}_{\bbeta^*,\sigma}$, $\ep^*=\by-\bX\bbeta^*$ is a vector of independent and identically distributed normal variables
with zero mean. Since $\sigma^*={|}\by-\bX\bbeta^*{|}/n^{1/2}$,
$z_j/\{(1-z_j^2)/(n-1)\}^{1/2}$ follows a $t$-distribution with $n-1$ degrees of freedom.
Lemma \ref{t-dist} with $m=n-1$ and $t^2=\log(p/\eps)>2$ implies
\bel{thm1-pf-t1}
\mathrm{pr}_{\bbeta^*,\sigma}\Big[\frac{(n-1)z_j^2}{1-z_j^2}>(n-1)\{e^{2t^2/(n-2)}-1\}\Big]
\le \frac{1+\eps_{n-1}}{\pi^{1/2}t}e^{-t^2}=\frac{(1+\eps_{n-1})\eps/p}{\{\pi\log(p/\eps)\}^{1/2}}.
\eel
Since $e^a-1\le \sum_{k=1}^\infty a^k/2^{k-1}=a/(1-a/2)$ for any $0<a<2$,
\bel{thm1-pf-t2}
(n-1)\{e^{2t^2/(n-2)}-1\}\le \frac{2(n-1)t^2/(n-2)}{1-t^2/(n-2)}\le \frac{2(n-1)t^2/n}{1-2t^2/n}.
\eel
The combination of (\ref{thm1-pf-t1}) and (\ref{thm1-pf-t2}) yields
\bes
\mathrm{pr}_{\bbeta^*,\sigma}\big[|z_j|>\{2\log(p/\eps)/n\}^{1/2}\big]
&=&\mathrm{pr}_{\bbeta^*,\sigma}\Big\{\frac{(n-1)z_j^2}{1-z_j^2}>\frac{2(n-1)t^2/n}{1-2t^2/n}\Big\}
\cr&\le&\mathrm{pr}_{\bbeta^*,\sigma}\Big\{\frac{(n-1)z_j^2}{1-z_j^2}>(n-1)(e^{\frac{2t^2}{n-2}}-1)\Big\}
\cr&\le& (1+\eps_{n-1})(\eps/p)/\{\pi\log(p/\eps)\}^{1/2}.
\ees
Since $\lam_0\ge \{(2/n)\log(p/\eps)\}^{1/2}(\xi+1)/\{(\xi-1)(1-\phi_-)\}$,
this bounds the tail probability of $z^*=\max_{j\le p}|z_j|$ by the union bound.
Since $n(\sigma^*/\sigma)^2$ follows the $\chi^2_n$ distribution,
$n^{1/2}(\sigma^*/\sigma-1)$ converges to $N(0,1/2)$ in distribution, which then implies
(\ref{th-2-2}) by (\ref{th-2-1}) under $\phi(\sigma)=o(n^{-1/2})$.
\end{proof}

\begin{proof}[of Theorem \ref{th-mleas}]
Let $h = \hbeta - \beta^*$. It follows from the proof of Theorem \ref{th-2} (i) that
$z^*\le (1-\phi_-)\lam_0(\xi-1)/(\xi+1)\le (\hsigma/\sigma^*)\lam_0(\xi-1)/(\xi+1)$, so that
$\hlam+z^*\sigma^* \le \xi(\hlam - z^*\sigma^*)$.
By (\ref{KKT}), $|x_j'Xh/n| = |x_j'(y - X\hbeta - \varepsilon^*)/n| \ge \hlam - z^*\sigma^*$ for $\hbeta_j\neq 0$.
Let $B\subseteq {\widehat S}\setminus S$ with $|B|\le m$.
Since $\kappa_+(m,S)$ is the upper sparse eigenvalue, $(\hlam - z^*\sigma^*)^2|B|\le \kappa_+(m,S)|Xh|_2^2/n$.
By the basic inequality (\ref{basic}) with $w=\beta^*$,
$|Xh|_2^2/n \le (\hlam+z^*\sigma^*)|h_S|_1$ and $h$ is in the cone $\scrC(\xi,S)$.
Thus, since $|h_S|_1^2\kappa^2(\xi,S)\le |Xh|_2^2|S|/n$ by (\ref{compatible}),
$|Xh|_2^2/n\le (\hlam+z^*\sigma^*)^2|S|/\kappa^2(\xi,S)$.
It follows that $|B|\le \kappa_+(m,S)\xi^2|S|/\kappa^2(\xi,S)<m$.
Since all $B\subseteq {\widehat S}\setminus S$ of size $|B|\le m$ have size $|B|<m$,
${\widehat S}\setminus S$ does not have a subset of size $m$.
This gives the first inequality in (\ref{th-mleas-1}).

Let $P_B$ be the orthogonal projection to the linear span of $(x_j,j\in B)$.
By the definition of $\sigma^*_{m,S}$ and the prediction error bound $|Xh|_2^2/n\le \eta_*(\hlam,\xi)$ in Theorem \ref{th-3},
\bes
\hsigma^2 - \hhsigma^2 = |P_{\widehat S}(y-X\hbeta)|_2^2/n
\le \big(|P_{\widehat S}\,\varepsilon^*|_2+|P_{\widehat S}Xh|_2\big)^2/n
\le \{\sigma^*_{m-1,S} + \surd \eta_*(\hlam,\xi)\}^2.
\ees
This gives the second inequality in (\ref{th-mleas-1}). The prediction error bound follows from
\bes
|X\hhbeta-X\beta^*|_2=|P_{\widehat S}y-X\beta^*|_2
\le |P_{\widehat S}(y-X\hbeta)|_2 + |Xh|_2 \le \{\sigma^*_{m-1,S} + 2 \surd \eta_*(\hlam,\xi)\}\surd n,
\ees
which implies the $\ell_2$ estimation error bound due to
$\kappa_-(m-1,S)|\hhbeta-\beta^*|_2^2 \le |X\hhbeta-X\beta^*|_2^2/n$.
Finally, the probability bound in (\ref{th-mleas-3}) follows directly from an application of the Gaussian
concentration inequality to the chi-squared variables in the union bound.
\end{proof}

\bibliographystyle{biometrika}
\bibliography{scaledlasso20120524}

\begin{thebibliography}{30}
\expandafter\ifx\csname natexlab\endcsname\relax\def\natexlab#1{#1}\fi

\bibitem[{Antoniadis(2010)}]{Antoniadis10}
\textsc{Antoniadis, A.} (2010).
\newblock Comments on: $\ell_1$-penalization for mixture regression models by
  {N}. {S}t\"{a}dler, {P}. {B}\"{u}hlmann and {S}. van de {G}eer.
\newblock \textit{Test} \textbf{19}, 257--258.

\bibitem[{Belloni et~al.(2011)Belloni, Chernozhukov \& Wang}]{BelloniCW11}
\textsc{Belloni, A.}, \textsc{Chernozhukov, V.} \& \textsc{Wang, L.} (2011).
\newblock Square-root lasso: Pivotal recovery of sparse signals via conic
  programming.
\newblock \textit{Biometrika} \textbf{98}, 791--806.

\bibitem[{Bickel et~al.(2009)Bickel, Ritov \& Tsybakov}]{BickelRT09}
\textsc{Bickel, P.}, \textsc{Ritov, Y.} \& \textsc{Tsybakov, A.} (2009).
\newblock Simultaneous analysis of lasso and {D}antzig selector.
\newblock \textit{Annals of Statistics} \textbf{37}, 1705--1732.

\bibitem[{Bunea et~al.(2007)Bunea, Tsybakov \& Wegkamp}]{BuneaTW07}
\textsc{Bunea, F.}, \textsc{Tsybakov, A.} \& \textsc{Wegkamp, M.~H.} (2007).
\newblock Sparsity oracle inequalities for the lasso.
\newblock \textit{Electronic Journal of Statistics} \textbf{1}, 169--194.

\bibitem[{Candes \& Tao(2007)}]{CandesT07}
\textsc{Candes, E.} \& \textsc{Tao, T.} (2007).
\newblock The {D}antzig selector: statistical estimation when $p$ is much
  larger than $n$ (with discussion).
\newblock \textit{Annals of Statistics} \textbf{35}, 2313--2404.

\bibitem[{Efron et~al.(2004)Efron, Hastie, Johnstone \&
  Tibshirani}]{EfronHJT04}
\textsc{Efron, B.}, \textsc{Hastie, T.}, \textsc{Johnstone, I.} \&
  \textsc{Tibshirani, R.} (2004).
\newblock Least angle regression (with discussion).
\newblock \textit{Annals of Statistics} \textbf{32}, 407--499.

\bibitem[{Fan et~al.(2012)Fan, Guo \& Hao}]{FanGH10}
\textsc{Fan, J.}, \textsc{Guo, S.} \& \textsc{Hao, N.} (2012).
\newblock Variance estimation using refitted cross-validation in ultrahigh
  dimensional regression.
\newblock \textit{Journal of Royal Statistical Society B} \textbf{74}, 37--65.

\bibitem[{Fan \& Li(2001)}]{FanL01}
\textsc{Fan, J.} \& \textsc{Li, R.} (2001).
\newblock Variable selection via nonconcave penalized likelihood and its oracle
  properties.
\newblock \textit{Journal of the American Statistical Association} \textbf{96},
  1348--1360.

\bibitem[{Fan \& Peng(2004)}]{FanP04}
\textsc{Fan, J.} \& \textsc{Peng, H.} (2004).
\newblock On non-concave penalized likelihood with diverging number of
  parameters.
\newblock \textit{Annals of Statistics} \textbf{32}, 928--961.

\bibitem[{Greenshtein(2006)}]{Greenshtein06}
\textsc{Greenshtein, E.} (2006).
\newblock Best subset selection, persistence in high-dimensional statistical
  learning and optimization under $\ell_1$ constraint.
\newblock \textit{Annals of Statistics} \textbf{34}, 2367--2386.

\bibitem[{Greenshtein \& Ritov(2004)}]{GreenshteinR04}
\textsc{Greenshtein, E.} \& \textsc{Ritov, Y.} (2004).
\newblock Persistence in high--dimensional linear predictor selection and the
  virtue of overparametrization.
\newblock \textit{Bernoulli} \textbf{10}, 971--988.

\bibitem[{Huang et~al.(2008)Huang, Ma \& Zhang}]{HuangMZ08}
\textsc{Huang, J.}, \textsc{Ma, S.} \& \textsc{Zhang, C.-H.} (2008).
\newblock Adaptive lasso for sparse high-dimensional regression models.
\newblock \textit{Statistica Sinica} \textbf{18}, 1603--1618.

\bibitem[{Huber \& Ronchetti(2009)}]{HuberR09}
\textsc{Huber, P.~J.} \& \textsc{Ronchetti, E.~M.} (2009).
\newblock \textit{Robust Statistics}.
\newblock Wiley, 2nd ed., pp. 172--175.

\bibitem[{Koltchinskii et~al.(2011)Koltchinskii, Lounici \&
  Tsybakov}]{KoltchinskiiLT10}
\textsc{Koltchinskii, V.}, \textsc{Lounici, K.} \& \textsc{Tsybakov, A.~B.}
  (2011).
\newblock Nuclear norm penalization and optimal rates for noisy low rank matrix
  completion.
\newblock \textit{Annals of Statistics} \textbf{39}, 2302--2329.

\bibitem[{Meinshausen \& B{\"u}hlmann(2006)}]{MeinshausenB06}
\textsc{Meinshausen, N.} \& \textsc{B{\"u}hlmann, P.} (2006).
\newblock High-dimensional graphs and variable selection with the lasso.
\newblock \textit{Annals of Statistics} \textbf{34}, 1436--1462.

\bibitem[{Meinshausen \& Buhlmann(2010)}]{MeinshausenB10}
\textsc{Meinshausen, N.} \& \textsc{Buhlmann, P.} (2010).
\newblock Stability selection (with discussion).
\newblock \textit{Journal of the Royal Statistical Society: Series B}
  \textbf{72}, 417--473.

\bibitem[{Meinshausen \& Yu(2009)}]{MeinshausenY09}
\textsc{Meinshausen, N.} \& \textsc{Yu, B.} (2009).
\newblock Lasso-type recovery of sparse representations for high-dimensional
  data.
\newblock \textit{Annals of Statistics} \textbf{37}, 246--270.

\bibitem[{Osborne et~al.(2000{\natexlab{a}})Osborne, Presnell \&
  Turlach}]{OsbornePT00a}
\textsc{Osborne, M.}, \textsc{Presnell, B.} \& \textsc{Turlach, B.}
  (2000{\natexlab{a}}).
\newblock A new approach to variable selection in least squares problems.
\newblock \textit{IMA Journal of Numerical Analysis} \textbf{20}, 389--404.

\bibitem[{Osborne et~al.(2000{\natexlab{b}})Osborne, Presnell \&
  Turlach}]{OsbornePT00b}
\textsc{Osborne, M.}, \textsc{Presnell, B.} \& \textsc{Turlach, B.}
  (2000{\natexlab{b}}).
\newblock On the lasso and its dual.
\newblock \textit{Journal of Computational and Graphical Statistics}
  \textbf{9}, 319--337.

\bibitem[{Scheetz et~al.(2006)Scheetz, Kim, Swiderski, Philp, Braun, Knudtson,
  Dorrance, DiBona, Huang, Casavant, Sheffield \& Stone}]{Scheetz06}
\textsc{Scheetz, T.~E.}, \textsc{Kim, K.-Y.~A.}, \textsc{Swiderski, R.~E.},
  \textsc{Philp, A.~R.}, \textsc{Braun, T.~A.}, \textsc{Knudtson, K.~L.},
  \textsc{Dorrance, A.~M.}, \textsc{DiBona, G.~F.}, \textsc{Huang, J.},
  \textsc{Casavant, T.~L.}, \textsc{Sheffield, V.~C.} \& \textsc{Stone, E.~M.}
  (2006).
\newblock Regulation of gene expression in the mammalian eye and its relevance
  to eye disease.
\newblock \textit{Proc. Nat. Acad. Sci} \textbf{103}, 14429--14434.

\bibitem[{She \& Owen(2011)}]{SheO11}
\textsc{She, Y.} \& \textsc{Owen, A.~B.} (2011).
\newblock Outlier detection using nonconvex penalized regression.
\newblock \textit{Journal of the American Statistical Association}
  \textbf{106}, 626--639.

\bibitem[{St\"{a}dler et~al.(2010)St\"{a}dler, B\"{u}hlmann \& van~de
  Geer}]{StadlerBG10}
\textsc{St\"{a}dler, N.}, \textsc{B\"{u}hlmann, P.} \& \textsc{van~de Geer, S.}
  (2010).
\newblock $\ell_1$-penalization for mixture regression models (with
  discussion).
\newblock \textit{Test} \textbf{19}, 209--285.

\bibitem[{Sun \& Zhang(2010)}]{SunZ10}
\textsc{Sun, T.} \& \textsc{Zhang, C.-H.} (2010).
\newblock Comments on: $\ell_1$-penalization for mixture regression models by
  {N}. {S}t\"{a}dler, {P}. {B}\"{u}hlmann and {S}. van de {G}eer.
\newblock \textit{Test} \textbf{19}, 270--275.

\bibitem[{van~de Geer(2008)}]{vandeGeer08}
\textsc{van~de Geer, S.} (2008).
\newblock High--dimensional generalized linear models and the lasso.
\newblock \textit{Annals of Statistics} \textbf{36}, 614--645.

\bibitem[{van~de Geer \& B{\"u}hlmann(2009)}]{vandeGeerB09}
\textsc{van~de Geer, S.} \& \textsc{B{\"u}hlmann, P.} (2009).
\newblock On the conditions used to prove oracle results for the lasso.
\newblock \textit{Electronic Journal of Statistics} \textbf{3}, 1360--1392.

\bibitem[{Ye \& Zhang(2010)}]{YeZ10}
\textsc{Ye, F.} \& \textsc{Zhang, C.-H.} (2010).
\newblock Rate minimaxity of the lasso and {D}antzig selector for the $\ell_q$
  loss in $\ell_r$ balls.
\newblock \textit{Journal of Machine Learning Research} \textbf{11},
  3481--3502.

\bibitem[{Zhang(2010)}]{Zhang10-mc+}
\textsc{Zhang, C.-H.} (2010).
\newblock Nearly unbiased variable selection under minimax concave penalty.
\newblock \textit{The Annals of Statistics} \textbf{38}, 894--942.

\bibitem[{Zhang \& Huang(2008)}]{ZhangH08}
\textsc{Zhang, C.-H.} \& \textsc{Huang, J.} (2008).
\newblock The sparsity and bias of the lasso selection in high-dimensional
  linear regression.
\newblock \textit{Annals of Statistics} \textbf{36}, 1567--1594.

\bibitem[{Zhang(2009)}]{Zhang09-l1}
\textsc{Zhang, T.} (2009).
\newblock Some sharp performance bounds for least squares regression with
  {$L_1$} regularization.
\newblock \textit{Ann. Statist.} \textbf{37}, 2109--2144.

\bibitem[{Zhao \& Yu(2006)}]{ZhaoY06}
\textsc{Zhao, P.} \& \textsc{Yu, B.} (2006).
\newblock On model selection consistency of lasso.
\newblock \textit{Journal of Machine Learning Research} \textbf{7}, 2541--2567.

\end{thebibliography}

\end{document}